\documentclass[11pt, letterpaper]{article}
\usepackage{arxiv}

\usepackage[utf8]{inputenc} 
\usepackage[T1]{fontenc}    
\usepackage{hyperref}       
\usepackage{url}            
\usepackage{booktabs}       
\usepackage{amsfonts}       
\usepackage{nicefrac}       
\usepackage{microtype}      
\usepackage{lipsum}		
\usepackage{graphicx}
\usepackage{natbib}
\usepackage{doi}

\usepackage[utf8]{inputenc} 
\usepackage[T1]{fontenc}    
\usepackage{hyperref}       
\usepackage{url}            
\usepackage{booktabs}       
\usepackage{amsfonts}       
\usepackage{nicefrac}       
\usepackage{microtype}      
\usepackage{xcolor}         

\usepackage{amsmath}
\usepackage{amssymb}
\usepackage{mathtools}
\usepackage{amsthm}

\usepackage{enumitem}
\usepackage[capitalize,noabbrev]{cleveref}
\usepackage{multirow}
\usepackage{graphicx}
\usepackage{subcaption}
\usepackage{wrapfig}

\theoremstyle{plain}
\newtheorem{theorem}{Theorem}[section]
\newtheorem{lemma}[theorem]{Lemma}
\theoremstyle{definition}

\theoremstyle{remark}

\title{On the Memorization of Consistency Distillation for Diffusion Models}

\author{
Bingqing Jiang\thanks{School of Computing \& Data Science, The University of Hong Kong. Email: bingqingjiang@connect.hku.hk}
\and
Difan Zou\thanks{School of Computing \& Data Science and Institute of Data Science, The University of Hong Kong. Email: dzou@hku.hk}
}

\begin{document}
\maketitle

\begin{abstract} 
Diffusion models are central to modern generative modeling, and understanding how they balance memorization and generalization is critical for reliable deployment. 
Recent work has shown that memorization in diffusion models is shaped by training dynamics, with generalization and memorization emerging at different stages of training.
However, deployed diffusion models are often further distilled, introducing an additional training phase whose impact on memorization is not well understood.
In this work, we analyze how distillation reshapes memorization behavior in diffusion models, taking consistency distillation as a representative framework. Empirically, we show that when applied to a teacher model that has memorized data, consistency distillation significantly reduces transferred memorization in the student while preserving, and sometimes improving, sample quality. To explain this behavior, we provide a theoretical analysis using a random feature neural network model \citep{bonnaire2025why}, showing that consistency distillation suppresses unstable feature directions associated with memorization while preserving stable, generalizable modes. Our findings suggest that distillation can serve not only as an acceleration tool, but also as a mechanism for improving the memorization--generalization trade-off.
\end{abstract}

\section{Introduction}

Diffusion models have become a central paradigm in modern generative modeling due to their strong empirical performance, stable training dynamics, and flexibility across data modalities~\citep{song2019generative,song2020improved,song2021scorebased,ho2020denoising,karras2022elucidating,song2021denoising}.
By modeling generation as a gradual denoising process, diffusion models achieve high sample fidelity and robust generalization, making them a cornerstone of modern generative systems~\citep{podell2024sdxl}.
Given their growing importance, understanding how diffusion models balance memorization and generalization has become a fundamental question~\citep{gu2025on,wen2024detecting,jeon2024understanding,li2023generalization,somepalli2023understanding}.
Recent studies show that memorization in diffusion models is governed by training
dynamics~\citep{bonnaire2025why,george2025denoising,pham2025memorization}.
In particular, models typically achieve high generative quality before
memorization emerges at later training stages, indicating that memorization is a
dynamic, time-dependent phenomenon.
This view is also supported by evidence that trained denoising scores are
smoother than closed-form empirical optima, which helps explain why practical
diffusion models can generalize instead of simply reproducing training
samples~\citep{wang2024discrepancy}.

However, existing analyses of memorization have largely focused on diffusion
models trained from scratch.
This setting does not fully capture modern deployment pipelines, where a
pretrained diffusion model is often further distilled to improve sampling
efficiency and reduce computational cost~\citep{pmlr-v202-song23a,salimans2022progressive,kim2024consistency,luo2023diff,geng2025consistency}.
This raises a critical question that has received little attention to date:
\begin{center}
\centering
\emph{How does distillation affect the memorization properties of diffusion models?}
\end{center}
At a high level, distillation is not merely a passive model compression step, but a new training process with its own objective, data distribution, and optimization dynamics~\citep{xiang2025dkdm}.
If memorization in diffusion models depends sensitively on training dynamics and time scales, as prior work suggests, then distillation may further reshape, suppress, or even amplify memorization inherited from the teacher.
Yet, the effect of distillation on memorization remains largely uncharacterized.


Recent progress toward few-step or even single-step diffusion generation has largely relied on distillation-based approaches~\citep{tee2024physics,luo2024one,geng2023one,yin2024one,starodubcev2025scale}, among which consistency models have emerged as a representative and widely adopted framework~\citep{pmlr-v202-song23a}.
In this work, we study this question in the context of consistency distillation.
Rather than explicitly supervising entire sampling trajectories or relying on large collections of teacher-generated samples~\citep{yin2024improved,park2025inference}, consistency-based methods train a student model via an additional optimization procedure that enforces local agreement between neighboring points along the probability flow ODE.
From the perspective of learning dynamics, consistency distillation introduces a nontrivial training phase beyond the original diffusion training.
This additional optimization stage operates under a distinct objective and data distribution, and can therefore alter the balance between memorization and generalization established during the teacher's training.
The main contributions can be summarized as follows:
\begin{itemize}[leftmargin=*, itemsep=1pt, topsep=1pt]
\item \textbf{Consistency distillation mitigates memorization.}    
    We show that consistency distillation can reduce memorization inherited by the student model, even when the teacher exhibits strong overfitting. This effect holds across a wide range of settings, demonstrating that distillation actively reshapes memorization behavior rather than passively inheriting it.
\item \textbf{Distillation can preserve or improve utility.}
    Beyond reducing memorization, we find that consistency distillation can also improve sample quality.
    When the teacher model operates in a moderate memorization regime, the distilled
    student can even surpass the teacher in generative performance, indicating that memorization reduction does not come at the expense of utility.
\item \textbf{A mechanism for reshaped training dynamics.}
We provide a mechanistic understanding of consistency distillation using a
tractable Random Feature Neural Network (RFNN) model.
Our analysis shows that consistency distillation reshapes training dynamics by
concentrating updates on statistically stable feature directions, while rendering
memorization-associated modes dynamically negligible.
This structured update dynamics preserves generalizable representations and explains the empirical reduction of memorization under consistency distillation.
\end{itemize}

\vspace{-3mm}
\section{Related Work}
\vspace{-2mm}

\paragraph{Memorization in Diffusion Models.}
Recent studies have investigated memorization in diffusion models~\citep{291199,wen2024detecting,zhang2025generalization}.
A key motivation is that the denoising score matching objective admits empirical
minimizers that reproduce training samples, implying that memorization is
theoretically expected in weakly regularized or small-data regimes~\citep{gu2025on,baptista2025memorization}.
Subsequent work shows that memorization and generalization undergo sharp
transitions as a function of dataset size, model capacity, and training dynamics,
including phase-transition and crossover phenomena~\citep{buchanan2025on,zeno2025when,pham2025memorization}.
A precise high-dimensional analysis is given by~\cite{george2025denoising},
who derive exact learning curves for diffusion models with random-feature
parameterizations.
Building on this framework,~\cite{bonnaire2025why} show that
diffusion training dynamics induce an implicit form of dynamical regularization,
creating a growing time window between the onset of generalization
and the emergence of memorization in overparameterized regimes.
Beyond direct duplication, recent extraction studies show that explicit or
surrogate conditioning can amplify memorization risks in diffusion
models~\citep{chen2025side}.
While prior work has primarily focused on diffusion models trained from scratch,
the memorization behavior of distilled diffusion models remains underexplored.
In this work, we analyze how the additional training stage introduced by
consistency distillation reshapes memorization relative to standard diffusion
training and provide a theoretical explanation for the observed behavior.

\paragraph{Consistency Distillation.}

Consistency distillation accelerates diffusion model sampling by enforcing
self-consistency across diffusion times, enabling efficient few-step generation
via distillation from pretrained diffusion models~\citep{pmlr-v202-song23a,lai2023on}.
Subsequent work extends this framework to improve quality--speed trade-offs,
stabilize training, and provide theoretical guarantees on estimation,
discretization, and convergence~\citep{kim2024consistency,NEURIPS2024_98a29475,
wang2025stable,dou2024theory,yang2025improved,chen2025convergence}.
However, these studies primarily emphasize efficiency and sample quality.
The effect of consistency distillation on memorization remains largely
unexplored, which is the focus of our study.

\section{Definitions and Preliminaries}

\subsection{Generative Score Matching}
\label{subsec:ddm}

Diffusion models define a generative mechanism by gradually
transforming data drawn from an unknown target distribution $P_0$ on
$\mathbb R^d$ into Gaussian noise through a continuous-time stochastic process.
A standard formulation uses the Ornstein--Uhlenbeck (OU) stochastic differential
equation
\begin{equation}
\mathrm d \mathbf{x}_t
=
- \mathbf{x}_t \,\mathrm dt
+ \sqrt{2}\,\mathrm d \mathbf{B}_t ,
\label{eq:ou_forward}
\end{equation}
where $\mathbf{B}_t$ denotes a standard Wiener process.
This forward diffusion induces a family of intermediate distributions
$\{P_t\}_{t\ge 0}$ that smoothly interpolate between $P_0$ and the standard
Gaussian distribution $\mathcal N(0,I_d)$ as $t \to \infty$.
The closed-form solution of (\ref{eq:ou_forward}) yields
\begin{equation}
\mathbf{x}_t
=
e^{-t} \mathbf{x}_0
+ \sqrt{\Delta_t}\,\boldsymbol{\xi},
\quad
\Delta_t = 1 - e^{-2t},
\label{eq:ou_solution}
\end{equation}
with $\boldsymbol{\xi} \sim \mathcal N(0,I_d)$ independent of $\mathbf{x}_0$.
Sampling from the target distribution is achieved by reversing the forward
process in time, which takes the form
\begin{equation}
-\mathrm d \mathbf{x}_t
=
\bigl[
\mathbf{x}_t + 2\nabla_{\mathbf{x}} \log P_t(\mathbf{x})
\bigr]\mathrm dt
+ \sqrt{2}\,\mathrm d \widetilde{\mathbf{B}}_t ,
\label{eq:reverse_sde}
\end{equation}
where $\widetilde{\mathbf{B}}_t$ is a Wiener process evolving backward in time, and
$\nabla_{\mathbf{x}} \log P_t(\mathbf{x})$
is the score function of the forward marginal at time $t$.
Then generation reduces to learning the score
$\nabla_{\mathbf{x}} \log P_t(\mathbf{x})$ for all relevant diffusion times.

The score function can be characterized as the minimizer of a denoising
score-matching objective.
In practice, the score is restricted to a parametrized family
$\{\mathbf{s}_\theta(\cdot,t)\}_{\theta}$, typically implemented by a neural network, and
the expectation over $P_0$ is replaced by an empirical average over a finite
training set $\{\mathbf{x}_\nu\}_{\nu=1}^n$~\citep{bonnaire2025why,vincent2011connection,hyvarinen2005estimation}:
\begin{equation}
\mathcal L_t(\theta)
=
\frac{1}{n}
\sum_{\nu=1}^n
\mathbb E_{\boldsymbol{\xi}\sim\mathcal N(0,I_d)}
\Bigl[
\bigl\|
\sqrt{\Delta_t}\, \mathbf{s}_\theta(\mathbf{x}_{\nu,t},t)
+ \boldsymbol{\xi}
\bigr\|_2^2
\Bigr],
\label{eq:score_matching_emp}
\end{equation}
where
$
\mathbf{x}_{\nu,t}
=
e^{-t}\mathbf{x}_\nu + \sqrt{\Delta_t}\boldsymbol{\xi}.
$

\subsection{Consistency Distillation}
\label{subsec:ou_forward_cd_objective}

Consistency distillation trains fast generative models by transferring
the local probability-flow dynamics of a pretrained diffusion model
into a time-consistent student mapping~\citep{pmlr-v202-song23a}.
The student is trained to produce consistent predictions along short segments
of the teacher-induced flow.
We adopt consistency distillation as our focus, and further rationale is given in
Appendix~\ref{rational_cd}.

Recall that the diffusion forward process admits an equivalent probability flow ODE
$
\frac{\mathrm d \mathbf{x}_t}{\mathrm d t}
=
h(t)\,\mathbf{x}_t
-
\frac{1}{2} g^2(t)\,\nabla_{\mathbf{x}} \log p_t(\mathbf{x}_t)$,
which generates the same marginal distributions $\{P_t\}$ as the forward SDE.
Given a pretrained teacher score model
$\mathbf{s}_\phi(\mathbf{x},t)\approx\nabla_{\mathbf{x}}\log p_t(\mathbf{x})$,
the PF-ODE induces a deterministic velocity field
$
\mathbf{v}_\phi(\mathbf{x},t)
=
h(t)\,\mathbf{x}
-
\frac{1}{2} g^2(t)\,\mathbf{s}_\phi(\mathbf{x},t)$.
Under the OU forward in~(\ref{eq:ou_forward}), the drift and diffusion coefficients are given by
$f(\mathbf{x},t)=-\mathbf{x}$ and $g(t)=\sqrt{2}$.
The associated probability flow ODE therefore admits the simplified form
\begin{equation}
\frac{\mathrm d \mathbf{x}_t}{\mathrm d t}
=
-\mathbf{x}_t
-\mathbf{s}_\phi(\mathbf{x}_t,t).
\label{eq:pf_ode_ou}
\end{equation}
In practice, the PF-ODE is discretized on a decreasing sequence of times
$T = t_0 > t_1 > \cdots > t_K = 0$.
Given a sample $\mathbf{x}_{t_{k+1}}$ at time $t_{k+1}$,
we adopt an explicit Euler ODE solver and obtain
the following single-step update:
\begin{equation}
\widehat{\mathbf{x}}^\phi_{t_k}
 = 
\mathbf{x}_{t_{k+1}}
 + 
(t_k - t_{k+1})
\Big(
 - \mathbf{x}_{t_{k+1}}
 - 
\mathbf{s}_\phi(\mathbf{x}_{t_{k+1}},t_{k+1}) 
\Big).
\label{eq:pf_ou_euler}
\end{equation}
We refer to $\widehat{\mathbf{x}}^\phi_{t_k}$ as the
\emph{teacher-induced one-step target}.
Let $f_{\boldsymbol{\theta}}(\mathbf{x},t)$ denote a student consistency model
that maps a noisy input $\mathbf{x}$ at time $t$
to a common representation, typically corresponding to an estimate of the clean data.
The objective of consistency distillation enforces
\emph{time consistency} across neighboring discretization points:
\begin{equation}
L_{\mathrm{CD}}(\boldsymbol{\theta})
=
\mathbb E_{k,\mathbf{x}_{t_{k+1}}}  
\left[ 
\left\|
f_\theta \big(\mathbf{x}_{t_{k+1}}, t_{k+1}\big)
-
f_\theta \big(\widehat{\mathbf{x}}^\phi_{t_k}, t_k\big) 
\right\|_2^2 
\right].
\label{eq:cd_loss_general}
\end{equation}
This objective transfers the local dynamics of the teacher PF-ODE
to the student without requiring the student to explicitly approximate the score function.

\section{Memorization and Generation Quality in Consistency Distillation}
\label{sec:cifar10}

\subsection{Experimental Setup}
\label{subsec:exp_setup}

\paragraph{Datasets.}
We evaluate consistency distillation in three settings: unconditional generation on CIFAR-10~\citep{krizhevsky2009learning}, class-conditional generation on ImageNet~\citep{imagenet}, and text-to-image generation based on Stable Diffusion v1.5~\citep{sd1_5}.
For CIFAR-10, we consider multiple reduced-data regimes by uniformly subsampling
$n \in \{3000,4000,5000,6000\}$ training images.
For ImageNet, we study class-conditional generation on two reduced-data subsets with 5k, 7k and 10k training images.
For text-to-image generation, we adopt a memorization-oriented protocol adapted from prior work~\citep{somepalli2023understanding,wen2024detecting}.
Specifically, we fine-tune Stable Diffusion v1.5 on a dataset consisting of 200 image--prompt pairs, each repeated 10 times to induce memorization, together with additional COCO image--prompt pairs to preserve generalization.
The amount of added COCO data is set to 3$\times$, 4$\times$, or 5$\times$ the size of the original 200-pair subset.
The resulting fine-tuned model serves as the teacher for subsequent consistency distillation.
\textit{Following prior memorization studies, no data augmentation is applied throughout to avoid ambiguity in memorization assessment}~\citep{gu2025on}.

\paragraph{Training configuration.}
For CIFAR-10 and ImageNet datasets, all distillation experiments use a pre-trained EDM diffusion model as the teacher~\citep{karras2022elucidating}. Student models are initialized from the same pre-trained diffusion models and fine-tuned via consistency distillation~\citep{pmlr-v202-song23a}. We use LPIPS~\citep{zhang2018unreasonable} as the metric in the consistency loss, using Heun's second-order solver with 18 discretization steps on CIFAR-10 and 40 discretization steps on ImageNet. 
In both settings, the distilled student is evaluated with 1-step generation unless otherwise noted.
For Stable Diffusion v1.5, we distill from the fine-tuned memorizing teacher using latent consistency distillation with Huber loss. Following the official LCM implementation~\citep{luo2023latent}, teacher targets are constructed with a DDIM-based ODE solver on a 50-step DDIM discretization of the original diffusion process, and the resulting student is evaluated with 4-step sampling.

\paragraph{Evaluation metrics.}
We measure generation quality by FID on CIFAR-10 and ImageNet, and by CLIP score on Stable Diffusion v1.5. To quantify memorization, we report both the standard $\ell_2$-based memorization ratio~\citep{yoon2023memorization} and an SSCD-based semantic metric on CIFAR-10, where under the $\ell_2$ criterion a sample is considered memorized if its distance to the nearest training image is less than one-third of that to the second nearest neighbor. On ImageNet, memorization is measured only in SSCD feature space, since pixel-space matching is less reliable for complex natural images~\citep{wen2024detecting}. For both CIFAR-10 and ImageNet, we use an SSCD threshold of $0.6$ to define memorization ratios and additionally report the $95$th-percentile (p95) similarity to reduce sensitivity to the threshold choice. For Stable Diffusion v1.5, memorization is assessed by SSCD-based statistics, including mean, maximum, and p95 similarity over generated samples. Additional implementation details are provided in Appendix~\ref{appendix_training_details} and~\ref{appendix_sd15_details}.

\subsection{Memorization under Matched Generation Quality}
\label{sec:similar_quality_mem}

We first ask whether the student memorizes less simply because it generates worse samples.
To avoid this confound, we compare teacher--student pairs with comparable FID or CLIP score and then examine their memorization behavior.
Table~\ref{tab:cifar10_mem_fid} reports the results for unconditional image generation on CIFAR-10.
Across all data settings, teacher and student achieve comparable FID, making the memorization comparison meaningful.
Under this matched-quality regime, the student consistently shows substantially lower memorization under both $l_2$- and SSCD-based criteria.
The upper tail of SSCD similarity is also markedly reduced, indicating that consistency distillation suppresses not only overall memorization but also the most severe near-duplicate generations.
Moving to class-conditional image generation, Table~\ref{tab:imagenet_mem_fid} shows the corresponding results on ImageNet.
Here, the student likewise remains close to the teacher in FID while exhibiting a much lower SSCD-based memorization ratio, showing that the reduction is not limited to the CIFAR-10 setting.
Finally, Table~\ref{tab:clip_sscd_extra} presents the Stable Diffusion v1.5 results under different extra-data settings.
In this text-to-image setting, teacher and student remain broadly comparable in CLIP score, while the student consistently reduces SSCD mean, maximum similarity, and upper-tail similarity.
This indicates that consistency distillation suppresses memorization both in overall similarity to training samples and in the extreme high-similarity cases most indicative of direct recall. 
Additional results on sampler effects, model capacity, and student--teacher architectural mismatch are provided in Appendix~\ref{Sampler_Effects},~\ref{sec:model_capacity}, and~\ref{sec:architecture_mismatch}, respectively.



\begin{table*}[t]
\footnotesize
\setlength{\tabcolsep}{4pt}
\centering
\begin{minipage}[t]{0.48\textwidth}
\centering

\captionof{table}{\textbf{Memorization and generation quality comparison on unconditional CIFAR-10 under different data settings}.}
\label{tab:cifar10_mem_fid}
\resizebox{\linewidth}{!}{%
\begin{tabular}{llccc}
\toprule
Setting & Model & FID & $l_2$ Mem & SSCD Mem / p95 \\
\midrule
\multirow{2}{*}{3000-data}
& Teacher & 21.78 & 4.88\% & 21.68\% / 0.8164 \\
& Student & 20.82 & 0.01\% & 2.83\% / 0.5623 \\
\midrule
\multirow{2}{*}{4000-data}
& Teacher & 21.52 & 3.73\% & 18.31\% / 0.7983 \\
& Student & 20.87 & 0.01\% & 0.67\% / 0.5128 \\
\midrule
\multirow{2}{*}{5000-data}
& Teacher & 23.82 & 5.94\% & 21.59\% / 0.8270 \\
& Student & 23.03 & 0.02\% & 1.66\% / 0.5342 \\
\midrule
\multirow{2}{*}{6000-data}
& Teacher & 24.57 & 4.35\% & 17.65\% / 0.8027 \\
& Student & 23.68 & 0.01\% & 0.37\% / 0.5010 \\
\bottomrule
\end{tabular}%
}

\end{minipage}
\hfill
\begin{minipage}[t]{0.48\textwidth}
\centering

\captionof{table}{\textbf{Memorization and generation quality comparison on class-conditional ImageNet under different data settings}.}
\label{tab:imagenet_mem_fid}
\begin{tabular}{llcc}
\toprule
Setting & Model & FID & SSCD Mem / p95 \\
\midrule
\multirow{2}{*}{5000-data}
& Teacher & 16.89 & 17.4\% / 0.7011 \\
& Student & 17.70 & 2.67\% / 0.5713 \\
\midrule
\multirow{2}{*}{7000-data}
& Teacher & 20.14 & 30.25\%/ 0.7309 \\
& Student & 28.65 & 2.65\% / 0.5552 \\
\midrule
\multirow{2}{*}{10000-data}
& Teacher & 27.38 & 18.57\% / 0.6972 \\
& Student & 21.56 & 1.68\% / 0.4852 \\
\bottomrule
\end{tabular}%
\end{minipage}
\end{table*}

\begin{table}[t]
\small
\centering
\caption{\textbf{Memorization and generation quality comparison on Stable Diffusion v1.5 under different extra-data settings}.}
\label{tab:clip_sscd_extra}
\begin{tabular}{llcccc}
\toprule
Setting & Model & CLIP & SSCD-mean & SSCD-max & SSCD-p95 \\
\midrule
\multirow{2}{*}{3$\times$ extra}
& Teacher & 0.2274 & 0.3612 & 0.6046 & 0.5451 \\
& Student & 0.2209 & 0.3380 & 0.5774 & 0.5064 \\
\midrule
\multirow{2}{*}{4$\times$ extra}
& Teacher & 0.2237 & 0.3390 & 0.5539 & 0.5109 \\
& Student & 0.2151 & 0.3002 & 0.5194 & 0.4795 \\
\midrule
\multirow{2}{*}{5$\times$ extra}
& Teacher & 0.2125 & 0.3003 & 0.6271 & 0.4764 \\
& Student & 0.2042 & 0.2685 & 0.5494 & 0.4204 \\
\bottomrule
\end{tabular}%
\end{table}

\subsection{Localized Memorization under Global Generalization}
\label{sec:localized_mem_global_gen}

We next evaluate consistency distillation in a more realistic regime where memorization is localized rather than global. 
In practice, memorization is often concentrated in a small subset of repeated or overexposed samples, while the model continues to generalize over the rest of the data distribution.
To reflect this structure, we construct localized memorization settings on CIFAR-10 and ImageNet by repeating only a small subset of classes and keeping the remaining classes non-repeated, thereby inducing global generalization with localized overexposure. 
Specifically, on CIFAR-10 we repeat 50 images from each of classes 0 and 1 five times and sample 1000 non-repeated images from classes 2--9, while on ImageNet we construct 50 repeated classes with 5 anchor images per class repeated 10 times and use 50 distinct non-repeated images for each of the remaining 950 classes.
As shown in Table~\ref{tab:realistic_localized_mem}, the teacher exhibits clear memorization on repeated classes but negligible memorization on non-repeated classes, confirming that memorization is indeed localized in this setting.
Under the same setup, the distilled student substantially reduces memorization on the repeated classes for both CIFAR-10 and ImageNet, while leaving the non-repeated classes essentially unchanged and maintaining comparable FID.
These results indicate that consistency distillation selectively suppresses localized memorization while preserving the broader generalized behavior of the model.

\begin{table}[t]
\centering
\caption{\textbf{Results under realistic localized memorization settings on CIFAR-10 and ImageNet}. Memorization statistics are reported as SSCD memorization ratio / p95 similarity to training data.}
\label{tab:realistic_localized_mem}
\small
\begin{tabular}{lccccc}
\toprule
Dataset & Model & FID & Repeated classes & Non-repeated classes & Overall \\
\midrule
\multirow{2}{*}{CIFAR-10}
& Teacher
& 12.02
& 74.54\% / 0.8983
& 0 / 0.4782
& 13.91\%/ 0.8255 \\
& Student
& 11.39
& 33.51\% / 0.7158
& 0 / 0.4747
& 5.69\%/0.6139 \\
\midrule
\multirow{2}{*}{ImageNet}
& Teacher
& 25.22
& 7.35\% / 0.6373
& 0 / 0.2308
& 0.37\% / 0.2497 \\
& Student
& 24.84
& 0 / 0.4118
& 0 / 0.2357
& 0 / 0.2416 \\
\bottomrule
\end{tabular}
\end{table}

\subsection{Student Training Dynamics during Consistency Distillation}
\label{subsec:moderate_mem}

\begin{figure*}[t]
  \centering
  \begin{subfigure}[t]{0.48\textwidth}
    \centering
    \includegraphics[width=\linewidth]{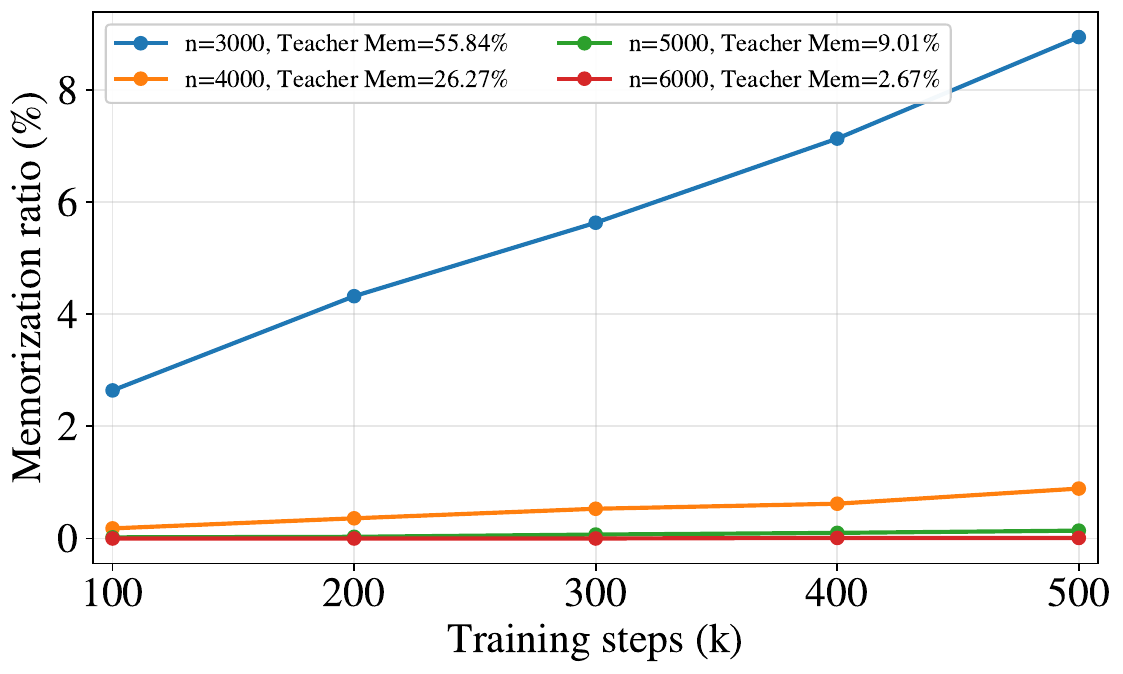}\vspace{-2mm}
    \label{fig:memorate_mem}
  \end{subfigure}
  \hfill
  \begin{subfigure}[t]{0.48\textwidth}
    \centering
    \includegraphics[width=\linewidth]{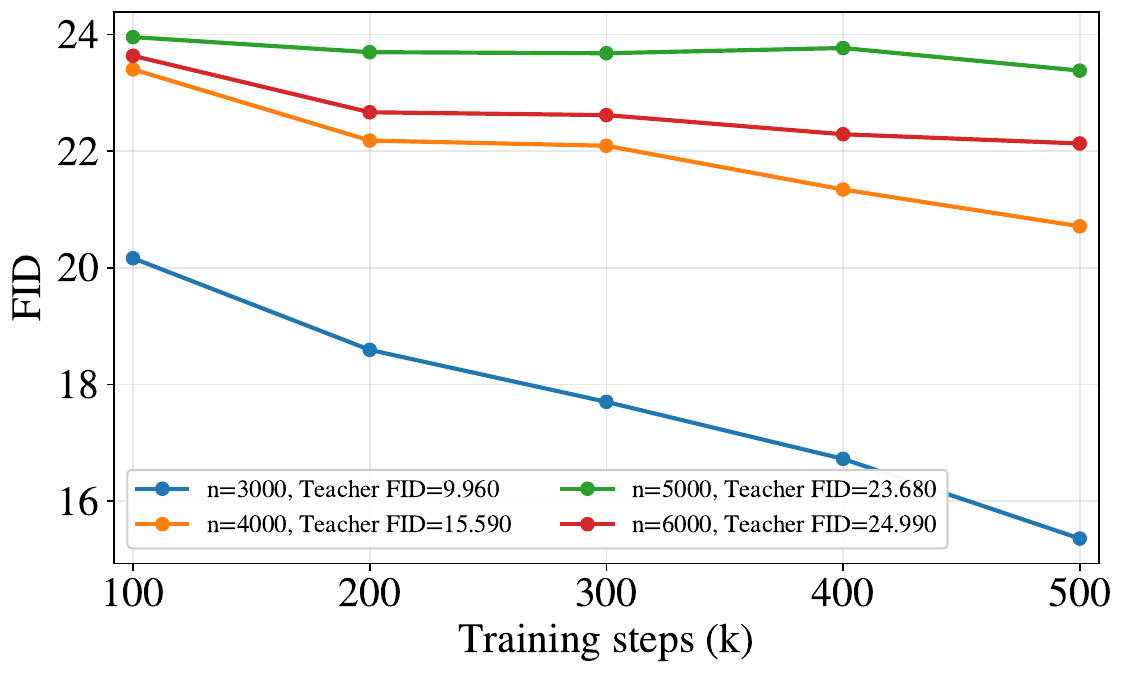}\vspace{-2mm}
    \label{fig:memorate_fid}
  \end{subfigure}
  \caption{
  \textbf{Student dynamics during consistency distillation on CIFAR-10.}
  \textbf{Left}: memorization ratio. \textbf{Right}: FID.
  Consistency distillation reliably reduces memorization, while FID improvement depends on the memorization state of the teacher.}\vspace{-2mm}
  \label{fig:student_dynamics_10mem}
\end{figure*}

Finally, we examine how memorization and generation quality evolve during the distillation process itself. 
In all settings, the EDM teacher is trained for 500k steps, and the consistency-distilled student is also trained for 500k steps for a fair comparison.
Figure~\ref{fig:student_dynamics_10mem} shows the student memorization ratio and FID as distillation proceeds. A consistent trend across all settings is that consistency distillation substantially suppresses memorization: although student memorization may increase gradually with distillation steps, it remains markedly below the teacher level throughout. 
At the same time, the FID behavior depends on the teacher regime. When the teacher memorization is relatively moderate, distillation not only reduces memorization but also yields students whose FID eventually surpasses that of the teacher. 
In contrast, when the teacher is already in a severely memorization-dominated regime, the student still exhibits much weaker memorization, but its FID remains worse than that of the teacher. 
In summary, consistency distillation reliably lowers memorization, while FID improves only when the teacher is not severely memorizing.
A theoretical explanation for the degraded student generation quality in the severe-teacher regime is provided in Appendix~\ref{sec:severe_mem}.

\section{Theoretical Analysis}
\label{sec:theory}

We now turn to the mechanism behind the empirical findings.
The goal is to explain why consistency distillation can reduce memorization while preserving the feature directions that support sample quality.

\subsection{One-step Consistency Objective.}

Our analysis considers a time-local regime with fixed $t'$ and $\Delta t\to 0$,
and focuses on a \emph{one-step} consistency distillation objective. While
consistency distillation is defined across multiple diffusion
times~\citep{pmlr-v202-song23a}, the one-step formulation isolates the
leading-order local consistency constraint induced by the teacher
probability-flow dynamics. It therefore provides a principled local
approximation to Eq.~(\ref{eq:cd_loss_general}), capturing the shared
leading-order update geometry underlying neighboring-step consistency
relations~\citep{george2025denoising,bonnaire2025why,li2025towards}. Under
this view, finite discretization mainly affects how accurately this local
constraint is realized in practice, while the leading-order mechanism is
already determined by the consistency objective itself. This interpretation is
also consistent with the discretization study in
Appendix~\ref{subsec:discretization_granularity}.
In this setting, the training objective compares the student outputs evaluated
at two nearby inputs connected by a single teacher-induced step:
\begin{equation}
L_{\mathrm{CD}}(\boldsymbol{\theta})
 = 
\frac{1}{n} \sum_{\nu=1}^{n}
\mathbb{E}_{\boldsymbol{\xi}}  
\left[ 
\left\| 
f_{\boldsymbol{\theta}} \big(\mathbf{x}_{\nu,t'}(\boldsymbol{\xi})\big)
 - 
f_{\boldsymbol{\theta}} \big(\widehat{\mathbf{x}}^\phi_{\nu,t}(\boldsymbol{\xi})\big) 
\right\|_2^2 
\right],
\label{eq:cd_loss_def}
\end{equation}
where $\mathbf{x}_{\nu,t}=e^{-t}\mathbf{x}_\nu + \sqrt{\Delta_t}\boldsymbol{\xi}$ and $\widehat{\mathbf{x}}^\phi_{\nu,t}(\boldsymbol{\xi})$ is the teacher-induced one-step target in~(\ref{eq:pf_ou_euler}).

\subsection{Random Feature Parameterization}

Following prior theoretical studies of diffusion learning dynamics and
representations~\citep{li2023generalization,bonnaire2025why,george2025denoising,han2025feature,han2025hidden}, we parameterize both the teacher and the student using an RFNN.
This tractable model makes the relevant feature geometry explicit while retaining the training-dynamics structure needed for our analysis.
An RFNN is a two-layer neural network in which the first-layer weights
$\mathbf{W} \in \mathbb{R}^{p\times d}$ are drawn i.i.d.\ from a Gaussian distribution and kept fixed, while only the second-layer weights are learned.
We work in an asymptotic regime where $d$, $p$, and $n$ jointly diverge to infinity, while the ratios $p/d$ and $n/d$ remain fixed.
The teacher and student share the same frozen random features matrix
$
\mathbf{W}_\phi = \mathbf{W}_\theta = \mathbf{W} \in \mathbb{R}^{p\times d}$,
$W_{ij} \stackrel{\text{i.i.d.}}{\sim} \mathcal{N}(0,1)$,
and use the same elementwise activation function $\sigma:\mathbb{R}\to\mathbb{R}$.
Define the feature map
$
\mathbf{h}(\mathbf{x})=
\sigma \left(\frac{\mathbf{W} \mathbf{x}}{\sqrt d}\right)
\in \mathbb{R}^p$.
Then the teacher score is modeled as
$
\mathbf{s}_\phi(\mathbf{x})
=
\frac{1}{\sqrt p}\, \mathbf{A}_\phi\, \mathbf{h}(\mathbf{x})$ with
$\mathbf{A}_\phi \in \mathbb{R}^{d\times p}$ is fixed,
while the student consistency mapping is parameterized as
$
\mathbf{f}_\theta(\mathbf{x})
=
\frac{1}{\sqrt p}\, \mathbf{B}_\theta\, \mathbf{h}(\mathbf{x})$ with
$\mathbf{B}_\theta \in \mathbb{R}^{d\times p}$ is trainable.
At fixed reference time $t'$ and step size $\Delta t$, the one-step distillation loss compares the student outputs evaluated at two nearby inputs generated by a teacher one-step update.
Let $
\Delta \mathbf{h}_\nu(\boldsymbol{\xi})
=
\mathbf{h}\big(\mathbf{x}_{\nu,t'}(\boldsymbol{\xi})\big)
-
\mathbf{h}\big(\widehat{\mathbf{x}}^\phi_{\nu,t}(\boldsymbol{\xi})\big)$
denote the feature increment induced by the teacher one-step update.
Under the RFNN parameterization, the output difference of the student model can be written as
$
\mathbf{f}_\theta \big(\mathbf{x}_{\nu,t'}(\boldsymbol{\xi})\big)
-
\mathbf{f}_\theta \big(\widehat{\mathbf{x}}^\phi_{\nu,t}(\boldsymbol{\xi})\big)
=
\frac{1}{\sqrt p}\, \mathbf{B}_\theta\, \Delta \mathbf{h}_\nu(\boldsymbol{\xi})$.
Substituting this expression into the one-step consistency distillation loss
(\ref{eq:cd_loss_def}), we obtain
\begin{align}
L_{\mathrm{CD}}(\mathbf{B}_\theta)
=
\frac{1}{n}\sum_{\nu=1}^{n}
\mathbb{E}_{\boldsymbol{\xi}}
\left[
\left\|
\frac{1}{\sqrt p}\, \mathbf{B}_\theta\, \Delta \mathbf{h}_\nu(\boldsymbol{\xi})
\right\|_2^2
\right]
=
\frac{1}{p}\,
\mathrm{Tr} \left(
\mathbf{B}_\theta\, \mathbf{U}_{\mathrm{cd}}\, \mathbf{B}_\theta^\top
\right),
\label{eq:cd_loss_quadratic}
\end{align}
where the consistency distillation curvature matrix
$\mathbf{U}_{\mathrm{cd}}\in\mathbb{R}^{p\times p}$ is defined as
\begin{equation}
\mathbf{U}_{\mathrm{cd}}
=
\frac{1}{n}\sum_{\nu=1}^{n}
\mathbb{E}_{\boldsymbol{\xi}}
\big[
\Delta \mathbf{h}_\nu(\boldsymbol{\xi})\,
\Delta \mathbf{h}_\nu(\boldsymbol{\xi})^\top
\big].
\label{eq:Ucd_def}
\end{equation}
Consequently, the spectrum of $\mathbf{U}_{\mathrm{cd}}$ fully characterizes the curvature of the one-step consistency distillation objective.

\subsection{Spectral Structure of Consistency Distillation}

We define the teacher-induced perturbation by
$
\boldsymbol{\delta x}(\mathbf{x},t',t) =
\widehat{\mathbf{x}}^\phi_t - \mathbf{x}_{t'}
=
\Delta t\, \mathbf{v}_\phi(\mathbf{x},t').
$
For notational simplicity, we will abbreviate
$\boldsymbol{\delta x}(\mathbf{x},t)$ as $\boldsymbol{\delta x}$ and $\mathbf{x_{t'}}$ as $\mathbf{x}$.
Under the RFNN parameterization, the objective of consistency distillation reduces to a quadratic form
whose curvature is determined by the second moment of the nonlinear feature
increment
$
\Delta \mathbf{h}
=
\mathbf{h}(\mathbf{x})
-
\mathbf{h}(\mathbf{x}+\boldsymbol{\delta x})$.
A central difficulty is that $\Delta \mathbf{h}$ is a nonlinear transformation
of random features.
To make this problem tractable in the high-dimensional regime, we adopt the
Gaussian equivalence principle for RFNN~\citep{george2025denoising,bonnaire2025why}, which replaces the
full feature interaction by an equivalent low-dimensional Gaussian
characterization.
Based on the assumptions in Appendix~\ref{assumption_appendix}, the following lemma characterizes the leading-order contribution to the curvature matrix $U_{cd}$ defined in (\ref{eq:Ucd_def}).

\begin{lemma}[Orthogonal second-moment decomposition of $\Delta \mathbf{h}$]
\label{lem:orth_decomp_delta_h}

Under Assumptions~\ref{ass:gaussian_equivalence} and
\ref{ass:activation_moment}, define the scalar coefficients
$
a_1(\mathbf{x},\boldsymbol{\delta x})=
\frac{
\mathbb E_{\zeta} \left[\Delta h_i\,\Delta g_i \mid \mathbf{x},\boldsymbol{\delta x}\right]
}{
\mathbb E_{\zeta} \left[\Delta g_i^2 \mid \mathbf{x},\boldsymbol{\delta x}\right]
}$ and
$a_0(\mathbf{x},\boldsymbol{\delta x})=
\frac{
\mathbb E_{\zeta} \left[\Delta h_i^2 \mid \mathbf{x},\boldsymbol{\delta x}\right]
}{
\mathbb E_{\zeta} \left[\Delta g_i^2 \mid \mathbf{x},\boldsymbol{\delta x}\right]
}
-
a_1(\mathbf{x},\boldsymbol{\delta x})^2$.
Then the conditional second moment of the feature increment admits the decomposition
\begin{equation}
\begin{aligned}
\mathbb E_{\zeta} \left[\Delta \mathbf{h}\Delta \mathbf{h}^\top    \mid   \mathbf{x},\boldsymbol{\delta x}\right]
  =
a_1(\mathbf{x},\boldsymbol{\delta x})^2
\mathbb E_{\zeta} \left[\Delta \mathbf{g}\Delta \mathbf{g}^\top      \mid   \mathbf{x},\boldsymbol{\delta x}\right]  
+& a_0(\mathbf{x},\boldsymbol{\delta x})
\mathbb E_{\zeta} \left[\Delta g_i^2   \mid \mathbf{x},\boldsymbol{\delta x}\right]\mathbf{I}_p,
\label{eq:delta_h_second_moment_decomp}
\end{aligned}
\end{equation}
where the conditional expectations
$\mathbb E_{\zeta}[\cdot \mid \mathbf{x},\boldsymbol{\delta x}]$
are taken with respect to an auxiliary Gaussian variable $\zeta$
representing the joint Gaussian law of the coordinate pairs
$(g_i,\Delta g_i)$ induced by a generic row $\mathbf{w}\sim\mathcal N(0,\mathbf{I}_d)$,
as specified in Assumption~\ref{ass:gaussian_equivalence}.

Assume further the small-noise one-step regime of
Assumption~\ref{ass:small_noise}.
In the isotropic setting $\boldsymbol{\Sigma}=\mathbf{I}_d$ and for the one-step
OU probability-flow update.
let $\gamma(t')$ and $\kappa(t')$ be the deterministic limits,
then $a_1(\mathbf{x},\boldsymbol{\delta x})$ and $a_0(\mathbf{x},\boldsymbol{\delta x})$
concentrate to deterministic limits
$a_1(t')$ and $a_0(t')$ given by~Eq.~(\ref{eq:a1}) and Eq.~(\ref{eq:a0}) in Appendix~\ref{lemma_proof}.
See Appendix~\ref{lemma_proof} for the proof.
\end{lemma}

Lemma~\ref{lem:orth_decomp_delta_h} implies that, to leading order, the only source of non-isotropic structure in $\mathbf{U}_{\mathrm{cd}}$ arises from the rank-one term $\Delta \mathbf{g}\, \Delta \mathbf{g}^\top$.
All remaining contributions collapse to an isotropic shift.
Consequently, the learning geometry induced by consistency distillation is entirely governed by how the teacher-induced perturbation $\boldsymbol{\delta x}$ is embedded into the random feature space through $\Delta \mathbf{g} = \mathbf{W}\boldsymbol{\delta x}/\sqrt d$.
To make this dependence explicit, we relate the one-step teacher update to the
geometry of the random feature space.
The following lemma recalls a closed-form characterization of the trained
teacher top layer, adapted from~\citep{bonnaire2025why}.

\begin{lemma}[\citep{bonnaire2025why}]
\label{lem:teacher_top_layer}

For RFNN score-matching trained by gradient flow with zero initialization at fixed
$t'$ on~(\ref{eq:score_matching_emp}), the converged teacher top layer satisfies
\begin{equation}
\mathbf{A}_\phi=
-\frac{\sqrt p}{\sqrt{\Delta_{t'}}}\,\mathbf{V}^\top\,\mathbf{U}^{-1},
\label{eq:Aphi_closed_form}
\end{equation}
where $
\mathbf{U}   =  
\frac{1}{n}\sum_{\nu=1}^n 
\mathbb E_\xi \left[ 
\sigma  \left(\frac{\mathbf{W}\mathbf{x}_{\nu,t'}(\boldsymbol{\xi})}{\sqrt d}\right)
\sigma  \left(\frac{\mathbf{W}\mathbf{x}_{\nu,t'}(\boldsymbol{\xi})}{\sqrt d}\right)^\top 
\right]$,
$
\mathbf{V}
 = 
\frac{1}{n}\sum_{\nu=1}^n 
\mathbb E_\xi \left[ 
\sigma  \left(\frac{\mathbf{W}\mathbf{x}_{\nu,t'}(\boldsymbol{\xi})}{\sqrt d}\right) 
\boldsymbol{\xi}^\top 
\right]$,
and $\Delta_{t'}$ is the OU forward variance at time $t'$.
Assume further that
the data distribution $P_x$ has zero mean, covariance
$\boldsymbol{\Sigma}=\mathbb E[\mathbf{x}\mathbf{x}^\top]$ with bounded spectrum and sub-Gaussian tails.
In the proportional high-dimensional limit
$n,p,d \to \infty$, $\psi_p = \frac{p}{d}$, $\psi_n = \frac{n}{d}$, the Gaussian equivalence results hold:
\begin{enumerate}[leftmargin=*,nosep]
\item
(\emph{Lemma C.1 in \citep{bonnaire2025why}})
The empirical spectral distribution of $\mathbf{U}$ coincides, in the large-dimensional limit, with that of the Gaussian-equivalent matrix
\begin{equation}
\mathbf{U}
=
\frac{1}{n}\mathbf{G} \mathbf{G}^\top+
b_{t'}^{2}\frac{\mathbf{W} \mathbf{W}^\top}{d}
+
s_{t'}^{2}\mathbf{I}_p ,
\label{eq:U_GEP_form}
\end{equation}
where $ \mathbf{G}  =  e^{-t'} a_{t'}\,\frac{\mathbf{W} \mathbf{X}'}{\sqrt d}  +  v_{t'}\,\boldsymbol{\Omega}$.
Here $\mathbf{X}'   \in   \mathbb R^{d\times n}$ has i.i.d.\ columns
$\mathbf{x}'_\nu   \sim   \mathcal N(0,\boldsymbol{\Sigma})$,
$\boldsymbol{\Omega}\in\mathbb R^{p\times n}$ has i.i.d.\ $\mathcal N(0,1)$ entries independent of
$(\mathbf{W},\mathbf{X}')$, and scalars
$a_{t'}, b_{t'}, v_{t'}, s_{t'}$ depend on $(t',\sigma,\boldsymbol{\Sigma})$.

\item
(\emph{Lemma C.4 in \citep{bonnaire2025why}})
The cross-covariance matrix $\mathbf{V}$ admits the deterministic equivalent
\begin{equation}
\mathbf{V}
=
\mu_1(t')\frac{\sqrt{\Delta_{t'}}}{\Gamma_{t'}}
\frac{\mathbf{W}}{\sqrt d},
\label{eq:V_GEP_form}
\end{equation}
where
$
\mu_1(t')
  =  
\mathbb E_{u\sim\mathcal N(0,1)} \big[\sigma(\Gamma_{t'}u)u\big], \Gamma_{t'}^{\,2}
  =   
e^{-2t'} \frac{\mathrm{Tr}(\boldsymbol{\Sigma})}{d}
  +  \Delta_{t'}$. By Gaussian integration by parts (Stein's lemma),
\(
\mu_1(t') = \Gamma_{t'}\,\mathbb E_{Z\sim\mathcal N(0,\Gamma_{t'}^{2})}[\sigma'(Z)].
\)
\end{enumerate}
\end{lemma}

Combining Lemma~\ref{lem:orth_decomp_delta_h} with the teacher characterization in Lemma~\ref{lem:teacher_top_layer}, we obtain the following structural characterization of the consistency distillation curvature.

\begin{theorem}
\label{thm:Ucd_structure_pfode_empirical}

Let $\{\mathbf{x}_\nu\}_{\nu=1}^n\subset\mathbb R^d$ be training samples with $\mathrm{Cov}(\mathbf{x}_\nu)=\boldsymbol{\Sigma}=\mathbf{I}_d$.
At a fixed diffusion time $t'>0$, draw $\boldsymbol{\xi}\sim\mathcal N(0,\mathbf{I}_d)$ and define the OU forward sample, and assume the small-noise one-step regime and the deterministic limits
$a_1(t')$, $a_0(t')$ from Lemma~\ref{lem:orth_decomp_delta_h}.
Then, as $\Delta t\to 0$,
\begin{equation}
\mathbf{U}_{\mathrm{cd}}
  =  
\Delta t^2 a_1(t')^2 
\Big( 
\mathbf{S}
 - \mu_1(t')^2\,\mathbf{S}\mathbf{U}^{-1}\mathbf{S}
 \Big) + \beta(t',\Delta t)\mathbf{I}_p,
\label{eq:Ucd_pfode_structure_emp}
\end{equation}
where $\mathbf{S} = \frac{\mathbf{W} \mathbf{W}^\top}{d}$ and the isotropic shift is
$
\beta(t',\Delta t)
= a_0(t')\,a_1(t')^2\,
\Delta t^2 t'^2\,
\frac{1}{d}\,\mathrm{Tr} \left(\mathbf{U}^{-1}\mathbf{S}\right)$.
See Appendix~\ref{the_proof} for the proof.
\end{theorem}

\begin{figure}[t]
  \centering
  \begin{subfigure}[t]{0.48\textwidth}
    \centering
    \includegraphics[width=\linewidth]{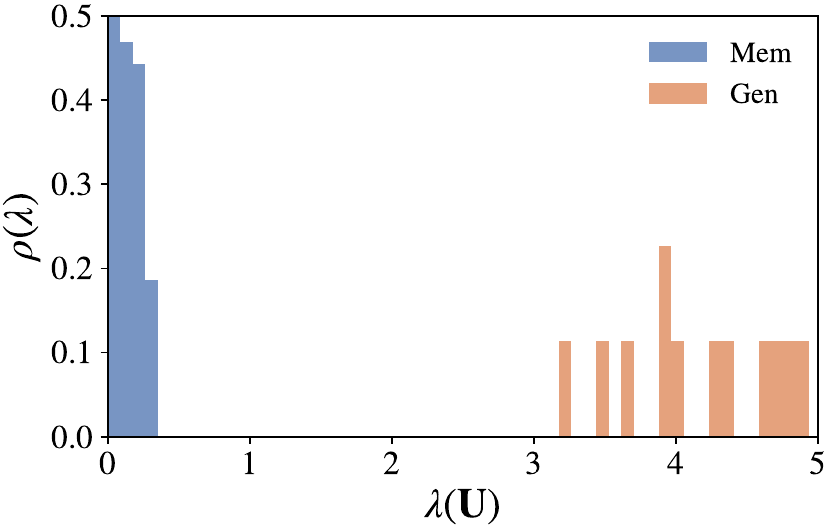}
    \caption{Spectral density of $\mathbf{U}$.}
    \label{fig:U_split_hist}
  \end{subfigure}
  \begin{subfigure}[t]{0.48\textwidth}
    \centering
    \includegraphics[width=\linewidth]{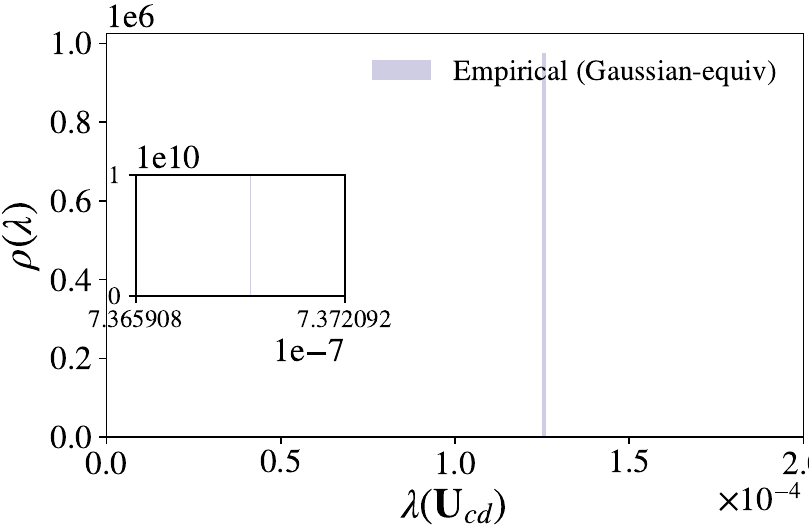}
    \caption{Spectral density of $\mathbf{U}_{\mathrm{cd}}$.}
    \label{fig:Ucd_spectrum_zoom}
  \end{subfigure}
 \caption{
\textbf{Spectral density of the teacher and consistency distillation curvature operators.}
\textbf{(a)} The teacher curvature operator $U$ exhibits a separated spectrum, with
low-eigenvalue modes associated with memorization and high-eigenvalue modes associated
with generalization.
\textbf{(b)} The consistency distillation curvature $U_{\mathrm{cd}}$ shows
sharp spectral atoms: a dominant spike at $\lambda=\beta$ induced by the isotropic
shift acting on the nullspace of $S$, while the remaining nontrivial eigenvalues are
concentrated in a low-dimensional subspace.
Both panels are computed with $\psi_p = 32$, $\psi_n = 4$, $t' = 0.01$, $\Delta t = 0.001$,
and $\rho_{\Sigma}(\lambda) = \delta(\lambda - 1)$. See Appendix~\ref{theory_exp} for details.} \vspace{-2mm}
  \label{fig:transfer_three}
\end{figure}

Theorem~\ref{thm:Ucd_structure_pfode_empirical} provides an explicit decomposition
of the curvature induced by one-step consistency distillation.
The resulting matrix $\mathbf{U}_{\mathrm{cd}}$ consists of two qualitatively
distinct components: an isotropic shift $\beta \mathbf{I}_p$, and a structured,
non-isotropic term $
\mathbf{A}=\mathbf{S}-\mu_1(t')^2\,\mathbf{S}\mathbf{U}^{-1}\mathbf{S}$.
Since $\mathbf{U}$ appears explicitly inside the structured term $\mathbf{A}$ via $\mathbf{U}^{-1}$, the teacher eigen-geometry provides the
natural coordinate system for understanding how consistency distillation redistributes curvature across memorization- and generalization-associated directions.
As established in~\citep{bonnaire2025why}, the empirical spectral density of $\mathbf{U}$ exhibits a characteristic two-bulk structure in the overparameterized regime.
This is visible in Fig.~\ref{fig:U_split_hist}, where modes with relatively large eigenvalues
$\lambda_i(\mathbf{U})$ align with generalization-dominated directions, whereas small eigenvalues correspond predominantly
to memorization-dominated directions.
This separation will be inherited by the structured deformation $\mathbf{A}=\mathbf{S}-\mu_1^2 \mathbf{S} \mathbf{U}^{-1}\mathbf{S}$ through the dependence on $\mathbf{U}^{-1}$.

We now turn to the empirical spectral density of $\mathbf{U}_{\mathrm{cd}}$ in
Fig.~\ref{fig:Ucd_spectrum_zoom}.
Theorem~\ref{thm:Ucd_structure_pfode_empirical} predicts a \emph{sharp spectral spike} induced by the isotropic term,
and the origin is purely algebraic: the random-feature Gram operator $\mathbf{S}=\frac{1}{d}\mathbf{W}\mathbf{W}^\top$ has rank at most $d$,
hence
$
\mathbb R^p = \ker(\mathbf{S})\oplus\mathrm{Im}(\mathbf{S}), \dim(\mathrm{Im}(\mathbf{S}))\le d\ll p$.
For $\mathbf{v}\in\ker(\mathbf{S})$, we have $\mathbf{S}\mathbf{v}=0$ and therefore $\mathbf{A}\mathbf{v}=0$, implying $\mathbf{U}_{\mathrm{cd}}\mathbf{v}=\beta \mathbf{v}$.
Thus, $\mathbf{U}_{\mathrm{cd}}$ has an eigenvalue exactly at $\lambda=\beta$ with multiplicity at least $p-d$,
which explains the prominent spike in Fig.~\ref{fig:Ucd_spectrum_zoom}.
Beyond this atom, all remaining eigenvalues are confined to the low-dimensional subspace $\mathrm{Im}(\mathbf{S})$, where
$
\left.\mathbf{U}_{\mathrm{cd}}\right|_{\mathrm{Im}(\mathbf{S})}
=
\beta \mathbf{I}
+
\Delta t^2 a_1(t')^2\,\left.\mathbf{A}\right|_{\mathrm{Im}(\mathbf{S})},
\dim(\mathrm{Im}(\mathbf{S}))\le d $.
Hence there are at most $d$ non-isotropic eigenvalues beyond the atom at $\beta$.
Moreover, because the prefactor $\Delta t^2 a_1(t')^2$ is of order $\Delta t^2$ and the two terms
in $\mathbf{A}=\mathbf{S}-\mu_1^2\mathbf{S}\mathbf{U}^{-1}\mathbf{S}$ can partially cancel within $\mathrm{Im}(\mathbf{S})$, the spectrum of the structured component
is typically highly compressed, appearing as a thin spike in Fig.~\ref{fig:Ucd_spectrum_zoom}.

\subsection{Empirical Validation of Spectral Filtering}
\label{sec:empirical_spectral_transfer}

To test this mechanism, we assess whether the non-isotropic consistency distillation term suppresses memorization-associated
directions while preserving those relevant for generalization.
We analyze its
action along the teacher eigenmodes
$\{\mathbf{u}_i\}_{i=1}^p$ of the curvature matrix $\mathbf{U}$.

\begin{figure*}[t]
  \centering
  \begin{subfigure}[t]{0.32\textwidth}
    \centering
    \includegraphics[width=\linewidth]{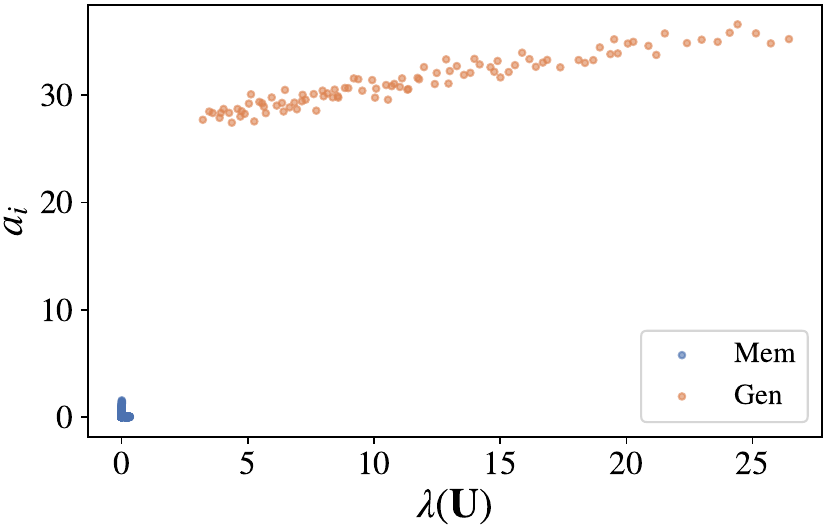}\vspace{-1mm}
    \caption{Visibility $a_i$.}
    \label{fig:a_vs_lambdaU}
  \end{subfigure}
  \hfill
  \begin{subfigure}[t]{0.32\textwidth}
    \centering
    \includegraphics[width=\linewidth]{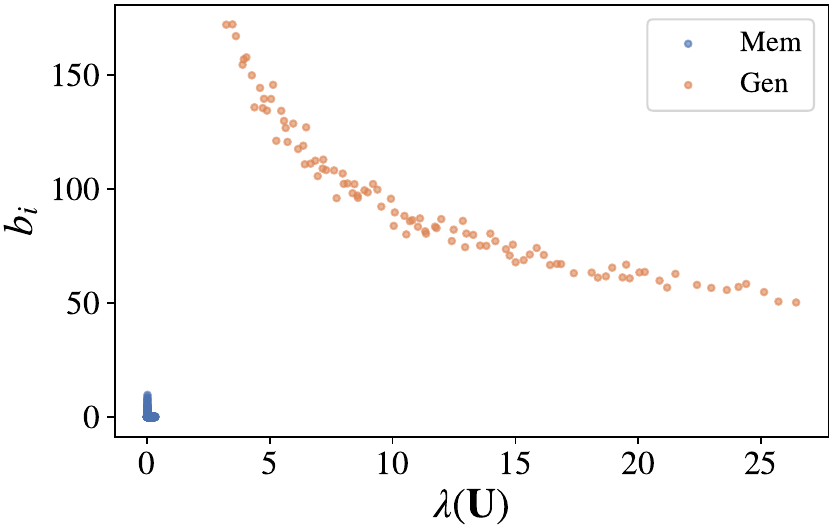}\vspace{-1mm}
    \caption{Resolvent term $b_i$.}
    \label{fig:b_vs_lambdaU}
  \end{subfigure}
    \hfill
  \begin{subfigure}[t]{0.32\textwidth}
    \centering
    \includegraphics[width=\linewidth]{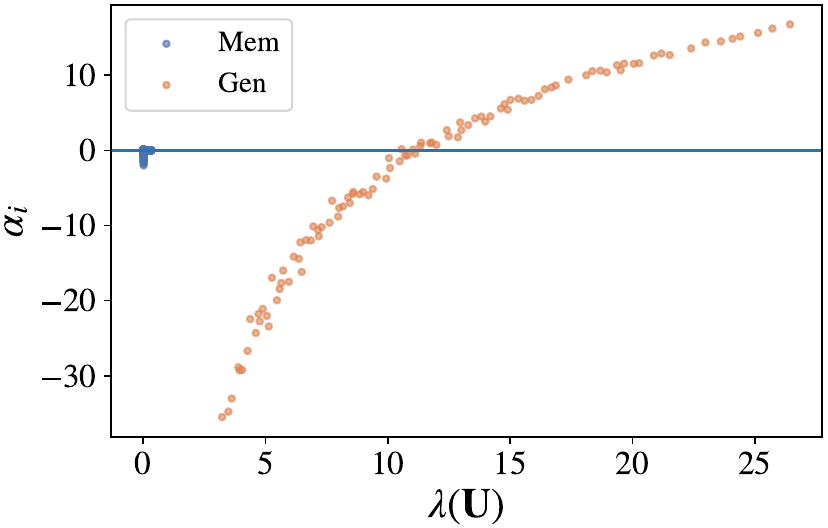}\vspace{-1mm}
    \caption{Signed response $\alpha_i$.}
    \label{fig:alpha_vs_lambdaU}
  \end{subfigure}
  \caption{
\textbf{Mode-wise spectral effects of non-isotropic consistency distillation.}
Each point corresponds to a teacher eigenmode $\mathbf{u}_i$, plotted against its
curvature eigenvalue $\lambda_i(\mathbf{U})$, with modes partitioned into
memorization-associated (Mem) and generalization-associated (Gen) subspaces.
\textbf{(a)} The visibility $a_i=\mathbf{u}_i^\top\mathbf{S}\mathbf{u}_i$ is uniformly
small for Mem modes and significantly larger for Gen modes.
\textbf{(b)} The resolvent term
$b_i=\mathbf{u}_i^\top\mathbf{S}\mathbf{U}^{-1}\mathbf{S}\mathbf{u}_i$ decreases with
$\lambda_i(\mathbf{U})$ within the Gen subspace.
\textbf{(c)} The resulting response
$\alpha_i=a_i-\mu_1(t')^2 b_i$ is negligible for Mem modes, while positive updates
concentrate in high-curvature Gen modes.
Results use $\psi_p=32$, $\psi_n=4$, $t'=0.01$, and
$\rho_{\Sigma}(\lambda)=\delta(\lambda-1)$.
}\vspace{-2mm}
\label{fig:transfer_three}
\end{figure*}

\noindent\textbf{Per-mode decomposition of the non-isotropic consistency distillation response.}
Recall that the leading non-isotropic component of the consistency distillation curvature takes the
form $\mathbf{A} = \mathbf{S}-\mu_1(t')^2\,\mathbf{S}\mathbf{U}^{-1}\mathbf{S}$.
For each teacher eigenmode $\mathbf{u}_i$, we introduce the quadratic forms
\begin{align}
a_i
=
\mathbf{u}_i^\top \mathbf{S}\,\mathbf{u}_i
\;\ge\;0, \quad
b_i
=
\mathbf{u}_i^\top \mathbf{S}\mathbf{U}^{-1}\mathbf{S}\,\mathbf{u}_i
=
(\mathbf{S}\mathbf{u}_i)^\top \mathbf{U}^{-1}(\mathbf{S}\mathbf{u}_i)
\;\ge\;0,
\end{align}
which respectively measure (i) the \emph{visibility} of mode $\mathbf{u}_i$ under
the random-feature metric $\mathbf{S}$, and (ii) the strength of its
\emph{resolvent-mediated subtraction} via $\mathbf{U}^{-1}$.
The resulting signed response along $\mathbf{u}_i$ is
\begin{equation}
\alpha_i
=
\mathbf{u}_i^\top \mathbf{A}\,\mathbf{u}_i
=
a_i
-
\mu_1(t')^2\, b_i,
\end{equation}
which quantifies the \emph{net non-isotropic consistency distillation update} assigned to that mode:
$\alpha_i<0$ corresponds to suppression, while $\alpha_i>0$ indicates net
retention.

\noindent\textbf{Mem/Gen partition.}
Following the established spectral geometry of $\mathbf{U}$, we classify modes
into memorization-associated subspaces (Mem) and generalization-associated subspaces (Gen) according to a fixed
threshold $\lambda_{\mathrm{th}}$, defined as
$
\mathcal{I}_{\mathrm{mem}}
=
\{i:\lambda_i(\mathbf{U})<\lambda_{\mathrm{th}}\},
\mathcal{I}_{\mathrm{gen}}
=
\{i:\lambda_i(\mathbf{U})\ge\lambda_{\mathrm{th}}\}$.

\noindent\textbf{(I) Visibility $a_i$ and resolvent structure $b_i$.}
Figures~\ref{fig:a_vs_lambdaU} and~\ref{fig:b_vs_lambdaU} visualize the two
constituent terms of $\alpha_i$.
Mem modes have uniformly small $a_i$, reflecting that these teacher eigenmodes are
highly sample-specific and largely orthogonal to the representational subspace
spanned by random features.
In contrast, Gen modes attain substantially larger $a_i$, as their smoother,
shared structure across samples aligns more strongly with the random-feature span,
leading to increasing visibility with $\lambda_i(\mathbf{U})$.
As for resolvent term $b_i$,
within the Gen subspace, lower-$\lambda$ modes incur larger $b_i$, reflecting
stronger subtraction induced by $\mathbf{U}^{-1}$, whereas higher-$\lambda$ Gen
modes are progressively less affected.
In contrast, Mem modes exhibit uniformly small $b_i$, as their weak alignment with
the random-feature subspace implies that $\mathbf{S}\mathbf{u}_i$ carries little
energy and is therefore minimally amplified by the resolvent.


\noindent\textbf{(II) Net per-mode response $\alpha_i$.}
Fig.~\ref{fig:alpha_vs_lambdaU} summarizes the net per-mode response $\alpha_i$
across the spectrum.
Mem modes are tightly concentrated near zero, and any isolated Mem modes with
$\alpha_i>0$ remain negligible due to their uniformly small visibility $a_i$.
In contrast, substantial positive responses occur almost exclusively within the
Gen subspace and increase with $\lambda_i(\mathbf{U})$.
Notably, the non-isotropic term is not uniformly enhancing within Gen:
modes near the lower edge of the Gen bulk are often suppressive ($\alpha_i<0$),
whereas the dominant positive contribution is carried by higher-curvature Gen
modes with larger $\lambda_i(\mathbf{U})$.
Consequently, even if some weaker Gen directions are attenuated, the effective
learning signal is governed by the high-eigenvalue Gen spectrum that encodes the
most consequential generative structure.
A more comprehensive analysis is provided in Appendix~\ref{sec:severe_mem}.

\noindent\textbf{(III) Global allocation of positive updates.}
To quantify how positive non-isotropic updates are distributed across subspaces,
we aggregate $\max(\alpha_i,0)$ and define
$
\mathrm{Share}_{\mathrm{mem}}^{+}
=
\frac{
\sum_{i\in\mathcal{I}_{\mathrm{mem}}}\max(\alpha_i,0)
}{
\sum_{i=1}^{p}\max(\alpha_i,0)
}$.
This metric captures the global allocation of positive curvature injection and is
robust to isolated atypical modes.
Empirically, we find $\mathrm{Share}_{\mathrm{mem}}^{+}\approx 9.5\times10^{-3}$,
indicating that nearly all positive non-isotropic updates are assigned to
generalization-associated directions.
Thus, at the level of global training dynamics, the non-isotropic component of
consistency distillation effectively shifts positive learning signal away from
memorization-associated modes.

\section{Conclusion and Discussion}

In summary, consistency distillation substantially reduces memorization in
diffusion models, including cases with strongly overfitted teacher models.
At the same time, sample quality is often preserved and can improve when the teacher exhibits a moderate level of memorization.
Our theoretical analysis suggests that this behavior arises because consistency
distillation actively reshapes the training geometry: it suppresses
memorization-associated directions while preserving generalization-relevant
updates, rather than passively inheriting teacher behavior.

\bibliographystyle{unsrtnat}
\bibliography{references}

\newpage
\appendix









\section{Rationale for Adopting Consistency Distillation}
\label{rational_cd}

\begin{figure}[t]
  \centering
  \begin{subfigure}[t]{0.23\textwidth}
    \centering
    \includegraphics[width=\linewidth]{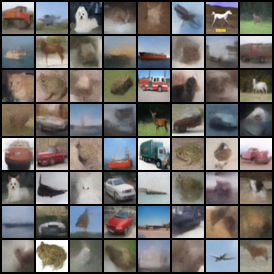}
    \caption{PD, 10\%-mem teacher}
  \end{subfigure}
  \hfill
  \begin{subfigure}[t]{0.23\textwidth}
    \centering
    \includegraphics[width=\linewidth]{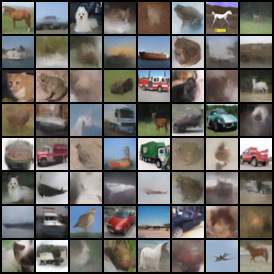}
    \caption{PD, 30\%-mem teacher}
  \end{subfigure}
  \hfill
  \begin{subfigure}[t]{0.23\textwidth}
    \centering
    \includegraphics[width=\linewidth]{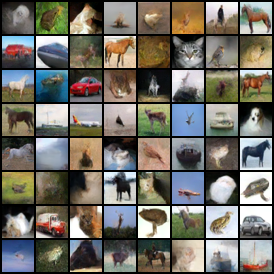}
    \caption{CD, 10\%-mem teacher}
  \end{subfigure}
  \hfill
  \begin{subfigure}[t]{0.23\textwidth}
    \centering
    \includegraphics[width=\linewidth]{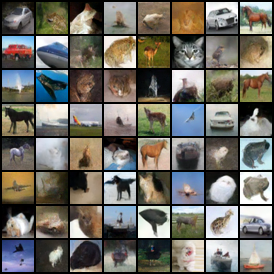}
    \caption{CD, 30\%-mem teacher}
  \end{subfigure}

  \caption{
    \textbf{Qualitative comparison between progressive distillation and consistency distillation under different teacher memorization levels trained on $6000$ data points}.
    PD samples exhibit noticeably degraded visual fidelity compared with CD,
    especially under limited data and higher teacher memorization.
  }
  \label{fig:pd_vs_cd_samples_6000data}
\end{figure}

\subsection{Motivation for Consistency Distillation}
\label{subsubsec:rationale_cd}

To motivate our choice of distillation framework, we first consider comparing with progressive distillation (PD)~\citep{salimans2022progressive}, a widely adopted approach for accelerating diffusion sampling.
PD iteratively reduces the number of sampling steps by training a student model to match the composition of multiple teacher DDIM steps, and has been shown to be effective in standard, data-rich regimes.
In each distillation stage, the student is initialized from the teacher and trained using a deterministic target constructed by composing two teacher DDIM transitions and analytically inverting a single student step.
This procedure is repeated while halving the sampling steps, yielding progressively faster samplers.

In these experiments, the teacher remains an EDM model trained on CIFAR-10, and we construct teachers with 10\% and 30\% memorization rates under the $l_2$ metric for both PD and CD distillation.
In our experiments, PD is implemented following the canonical discrete-time formulation.
The student time index is sampled from a fixed discrete grid, and training proceeds by matching one student step to two teacher steps.
We train PD models for $600{,}000$ iterations with batch size $64$, Adam optimization (weight decay $0$), gradient clipping $1.0$, and a linearly decayed learning rate.
Unless otherwise specified, we employ the perceptual LPIPS loss, which is commonly used to stabilize distillation under aggressive step reduction.

A critical hyperparameter in PD is the \emph{initial noise resolution}, controlled by \texttt{start\_scales}.
Setting \texttt{start\_scales=4096} corresponds to initializing distillation from a very fine-grained discretization of the diffusion process, i.e., a large number of teacher sampling steps.
This choice ensures that the initial teacher accurately resolves high-noise dynamics and provides well-defined multi-step trajectories for the student to imitate.
While such a setting is computationally demanding, it represents a favorable configuration for PD and is commonly adopted to avoid compounding discretization errors in early distillation stages.

\paragraph{Empirical fragility of PD under limited data and memorizing teachers.}
Despite this favorable configuration, PD exhibits limited robustness in the regime we consider.
With $6000$ training samples, the final PD model distilled from a $10\%$-memorization teacher attains an FID of $45.42$ with a memorization ratio of $0.79\%$.
When distilled from a $30\%$-memorization teacher, the final PD model attains an FID of $44.05$ with a memorization ratio of $3.74\%$.
These results indicate a substantial degradation in sample quality, even though the memorization ratios remain moderate.

This degradation is not merely quantitative.
As shown in left two panels of Fig.~\ref{fig:pd_vs_cd_samples_6000data}, PD samples under both $10\%$ and $30\%$ memorization teachers exhibit visibly reduced visual fidelity, including blurred structures and weakened object coherence.
Notably, this behavior persists despite careful hyperparameter choices and a large initial discretization scale, suggesting that PD is sensitive to the combined effects of limited data, teacher memorization, and aggressive step reduction.

\paragraph{Extension beyond trajectory-based distillation.}
While PD is the primary non-CD baseline in our study because it is the most directly comparable distillation method to CD, we also evaluate DMD to examine whether the memorization-suppression effect extends beyond this trajectory-based setting~\citep{dmd}.
We find that DMD likewise exhibits reduced memorization relative to the teacher, suggesting that the effect is not unique to CD.
At the same time, its raw sample quality is substantially worse.
To distinguish genuine memorization reduction from trivial degradation, we further refine both PD and DMD outputs using the same pretrained EDM prior by perturbing the one-step outputs to a high noise level and reconstructing them with the standard EDM PF-ODE solver~\citep{karras2022elucidating}.
After this prior-guided reconstruction, sample quality improves substantially for both methods, while memorization remains clearly below teacher levels.
These results suggest that the observed memorization reduction is not merely a byproduct of poor raw sample quality.

\begin{table}[t]
\centering
\caption{\textbf{Results for PD and DMD before and after EDM-based refinement under limited-data CIFAR-10 settings}. Refinement substantially improves sample quality while memorization remains well below teacher levels.}
\label{tab:pd_dmd_refine}
\begin{tabular}{llcccc}
\toprule
Method Group & Model & FID  & $l_2$ Mem  & SSCD Mem  & SSCD p95  \\
\midrule
\multirow{3}{*}{PD}
& Teacher    & 22.68 & 10.00\% & 26.93\% & 0.8586 \\
& PD         & 43.66 & 0.46\%  & 3.63\%  & 0.5664 \\
& PD+refine  & 10.88 & 0.00\%  & 0.32\%  & 0.4822 \\
\midrule
\multirow{3}{*}{DMD}
& Teacher    & 14.29 & 26.57\% & 49.12\% & 0.9220 \\
& DMD        & 46.91 & 3.93\%  & 7.17\%  & 0.6390 \\
& DMD+refine & 15.59 & 0.55\%  & 3.71\%  & 0.5617 \\
\bottomrule
\end{tabular}
\end{table}

\paragraph{Why we focus on consistency distillation.}
In contrast, consistency distillation demonstrates markedly stronger robustness in the same regime.
Rather than matching composed multi-step trajectories, consistency distillation trains the student to produce self-consistent predictions across neighboring noise levels.
This objective avoids explicit inversion of multi-step DDIM transitions and reduces the dependence on long teacher trajectories, which can be particularly fragile when the teacher exhibits memorization or when data is scarce.

Empirically, this robustness translates into a substantially improved quality--memorization tradeoff.
With $6000$ training samples, CD distilled from a $10\%$-memorization teacher achieves an FID of $21.19$ with a memorization ratio of $0.56\%$.
Under a $30\%$-memorization teacher, CD (two-step) achieves an FID of $20.60$ with a memorization ratio of $3.17\%$.
As illustrated in right two panels of Fig.~\ref{fig:pd_vs_cd_samples_6000data}, CD preserves significantly higher visual fidelity than PD under both teacher settings, even when the teacher itself exhibits substantial memorization.


\subsection{Motivation for the Discrete Formulation of Consistency Distillation}

\begin{figure}[t]
  \centering
  \begin{subfigure}[t]{0.23\textwidth}
    \includegraphics[width=\linewidth]{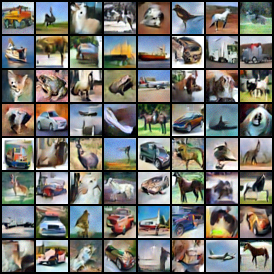}
    \caption{Continuous CD, 10\%-mem teacher}\label{10-continue}
  \end{subfigure}
  \hfill
  \begin{subfigure}[t]{0.23\textwidth}
    \includegraphics[width=\linewidth]{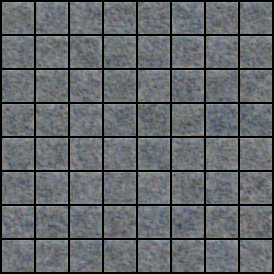}
    \caption{Continuous CD, 30\%-mem teacher}\label{30-continue}
  \end{subfigure}
  \hfill
  \begin{subfigure}[t]{0.23\textwidth}
    \includegraphics[width=\linewidth]{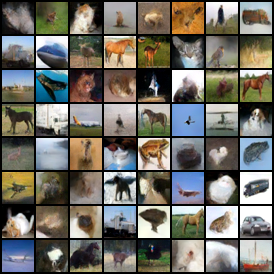}
    \caption{Discrete CD, 10\%-mem teacher}\label{10-lisan}
  \end{subfigure}
  \hfill
  \begin{subfigure}[t]{0.23\textwidth}
    \includegraphics[width=\linewidth]{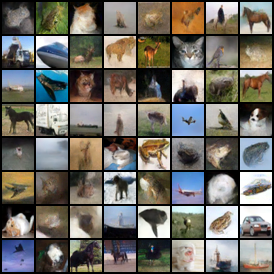}
    \caption{Discrete CD, 30\%-mem teacher}\label{30-lisan}
  \end{subfigure}

  \caption{
    \textbf{Comparison between continuous-time and discrete consistency distillation under limited data}.
    Continuous-time CD exhibits blurred or failed generations, while the discrete objective remains stable under the same training budget.
  }
  \label{fig:cd_continuous_vs_discrete_5000data}
\end{figure}

The original formulation of consistency distillation admits a continuous-time extension, obtained as the infinite-step limit of the discrete objective~\citep{pmlr-v202-song23a}.
Under suitable smoothness assumptions on the consistency function, the metric, and the teacher score, the rescaled discrete consistency loss converges to a continuous-time objective defined along the probability flow ODE.
In the commonly used stop-gradient setting, this limit yields a \emph{pseudo-objective} whose gradient matches that of the discrete loss as the number of time steps tends to infinity.

Concretely, letting $f_\theta(x_t, t)$ denote the student consistency model and $s_\phi(x_t, t)$ the teacher score, the continuous-time consistency distillation objective takes the form
\begin{equation}
\label{eq:cd_continuous_loss}
\mathcal{L}_{\mathrm{CD}}^{\mathrm{cont}}(\theta)
\;=\;
\mathbb{E}_{x \sim p_{\mathrm{data}},\, t}
 \left[
\lambda(t)\,
\Big\|
\partial_t f_\theta(x_t, t)
-
t\, \nabla_x f_\theta(x_t, t)\, s_\phi(x_t, t)
\Big\|^2
\right],
\end{equation}
where $x_t \sim \mathcal{N}(x, t^2 I)$ and $\lambda(t)$ is a bounded weighting function.
This objective is minimized if and only if the student model matches the ground-truth consistency function induced by the teacher probability flow.

While Eq.~\eqref{eq:cd_continuous_loss} provides a principled characterization of consistency distillation in the continuous-time limit, it introduces nontrivial practical challenges.
At finite training resolution, optimizing this objective requires implicitly estimating directional derivatives of the student network along the teacher-induced flow, which involves Jacobian--vector products.
Such derivative-based signals are sensitive to model curvature and to noise in the teacher score, and this sensitivity is amplified when the teacher exhibits memorization.

These effects become particularly pronounced in low-data regimes.
With limited data, the student observes fewer distinct diffusion trajectories, and the continuous-time objective aggregates derivative information along these trajectories, increasing variance and compounding optimization noise.
Empirically, this leads to unstable optimization behavior: under $5000$ training samples, the continuous-time consistency objective already produces noticeably blurred generations under a $10\%$-memorization teacher in Fig.~\ref{10-continue}, and fails to generate coherent samples under a $30\%$-memorization teacher in Fig.~\ref{30-continue}.

In contrast, the discrete consistency distillation objective enforces consistency through explicit, finite differences between model predictions at neighboring noise levels.
Each update depends only on forward evaluations of the student model at a small number of discrete time points, avoiding the need to estimate derivatives along the probability flow.
As a result, the discrete objective is less sensitive to high-curvature, memorization-associated directions in the teacher dynamics.
Under identical data scale, architecture, and training budget, the discrete formulation yields stable and visually coherent samples for both $10\%$ and $30\%$-memorization teachers in Figs.~\ref{10-lisan} and~\ref{30-lisan}.

\section{Detailed Experimental Setup}
\label{detail_exp}

\subsection{Additional Experimental Details for CIFAR-10 and ImageNet}
\label{appendix_training_details}

\paragraph{Backbone architecture.}
For CIFAR-10, we use the NCSN++ backbone~\citep{song2021scorebased}, following
the standard architectural choice in prior consistency distillation work~\citep{pmlr-v202-song23a}.
For ImageNet, we use the class-conditional DDPM++ architecture~\citep{dhariwal2021diffusion}, consistent with the ImageNet
setup in Consistency Models~\citep{pmlr-v202-song23a}.
In both cases, we employ the same backbone family for the EDM teachers and the
corresponding consistency models, so architectural differences do not confound
comparisons.

\paragraph{Consistency model parameterization and boundary condition.}
We represent a consistency model as
\begin{equation}
f_\theta(x,t) = c_{\mathrm{skip}}(t)\,x + c_{\mathrm{out}}(t)\,F_\theta(x,t),
\end{equation}
where the coefficients are chosen to satisfy the boundary constraint at the
minimum time $\varepsilon$.
Using $\sigma_{\mathrm{data}}=0.5$, we set
\begin{equation}
c_{\mathrm{skip}}(t)
=
\frac{\sigma_{\mathrm{data}}^2}{(t-\varepsilon)^2+\sigma_{\mathrm{data}}^2},
\qquad
c_{\mathrm{out}}(t)
=
\frac{\sigma_{\mathrm{data}}(t-\varepsilon)}{\sqrt{\sigma_{\mathrm{data}}^2+t^2}},
\end{equation}
which guarantees $c_{\mathrm{skip}}(\varepsilon)=1$ and
$c_{\mathrm{out}}(\varepsilon)=0$.
This follows the consistency-model parameterization in~\citep{pmlr-v202-song23a}
and ensures the required boundary condition when $\varepsilon>0$.

\paragraph{Consistency distillation and optimization.}
For consistency distillation, the student network is initialized from the
corresponding pretrained EDM weights.
Training uses Rectified Adam (RAdam) with no warm-up, no learning-rate decay,
and no weight decay.
We maintain an exponential moving average (EMA) of the online model parameters,
consistent with the setup in~\citep{pmlr-v202-song23a,karras2022elucidating}.
No data augmentation or additional regularization is used, in order to avoid potential confounding effects in the assessment of memorization~\citep{gu2025on}.

\paragraph{Schedules for consistency training.}
When training consistency models, we use a time-step schedule $N(k)$ and an EMA
schedule $\mu(k)$ of the form
\begin{align}
N(k)
&=
\Big\lfloor
\frac{c\,k}{K}\bigl((s_1+1)^2-s_0^2\bigr) + s_0^2
\Big\rfloor - 1,
\\
\mu(k)
&=
\exp \left(
\frac{s_0 \log \mu_0}{N(k)}
\right),
\end{align}
where $k\in\{1,\dots,K\}$ indexes training iterations, $K$ is the total number
of iterations, $s_0$ and $s_1$ are the starting and ending discretization
budgets, and $\mu_0$ is the initial EMA decay factor.

\subsection{Additional Experimental Details for Stable Diffusion v1.5}
\label{appendix_sd15_details}

\paragraph{Latent-space formulation.}
Following Latent Consistency Models~\citep{luo2023latent}, we perform consistency
distillation in the latent space rather than in pixel space.
Let $x$ denote an image and $c$ its text prompt.
Using the pretrained VAE encoder of Stable Diffusion v1.5, we first map the image into latent space,
\begin{equation}
z = E(x),
\end{equation}
where $E(\cdot)$ is the fixed encoder.
Distillation is then performed entirely in latent space rather than pixel space.

Following the standard latent diffusion parameterization, a noisy latent at time $t$ is written as
\begin{equation}
z_t = \alpha_t z + \sigma_t \epsilon, \qquad \epsilon \sim \mathcal{N}(0,I),
\end{equation}
where $(\alpha_t,\sigma_t)$ are determined by the diffusion noise schedule.

\paragraph{Consistency model parameterization.}
We parameterize the latent consistency model as
\begin{equation}
f_{\theta}(z_t,c,t)
=
c_{\mathrm{skip}}(t)\, z_t
+
c_{\mathrm{out}}(t)
\left(
\frac{z_t-\sigma_t \hat{\epsilon}_{\theta}(z_t,c,t)}{\alpha_t}
\right),
\label{eq:sd15_lcm_param}
\end{equation}
where $\hat{\epsilon}_{\theta}$ is the student noise predictor.
This is the standard $\epsilon$-prediction parameterization used in LCM-style distillation for latent diffusion models.
In our experiments, the student is initialized from the fine-tuned Stable Diffusion v1.5 teacher before consistency distillation.

\paragraph{Teacher trajectory construction.}
To construct distillation targets, we follow the official LCM implementation and approximate the PF-ODE trajectory of the teacher with a DDIM-based solver.
Given a discretization pair $(t_{n+1}, t_n)$ on a 50-step DDIM grid, we first sample
\begin{equation}
z_{t_{n+1}} = \alpha_{t_{n+1}} z + \sigma_{t_{n+1}} \epsilon,
\qquad \epsilon \sim \mathcal{N}(0,I),
\end{equation}
and then obtain the previous latent by a DDIM ODE step,
\begin{equation}
\hat{z}^{\Psi}_{t_n}
=
\Psi_{\mathrm{DDIM}}(z_{t_{n+1}}, t_{n+1}, t_n, c),
\label{eq:sd15_ddim_solver}
\end{equation}
where $\Psi_{\mathrm{DDIM}}$ denotes the DDIM-based solver applied to the teacher model.
Thus, the target is defined by an earlier point on the same reverse-time teacher trajectory, rather than by matching the teacher prediction at exactly the same noisy input.

\paragraph{Consistency distillation objective.}
Let $\theta^{-}$ denote the EMA parameters of the student.
We optimize the latent consistency objective
\begin{equation}
\mathcal{L}_{\mathrm{LCD}}
=
\mathbb{E}_{z,c,n,\epsilon}
\left[
\ell_{\mathrm{Huber}}
\Big(
f_{\theta}(z_{t_{n+1}},c,t_{n+1}),
\,
f_{\theta^{-}}(\hat{z}^{\Psi}_{t_n},c,t_n)
\Big)
\right],
\label{eq:sd15_lcd_loss}
\end{equation}
where $\ell_{\mathrm{Huber}}(\cdot,\cdot)$ is the Huber loss.
No auxiliary reconstruction, perceptual, or adversarial losses are introduced.

\paragraph{Components kept fixed.}
Throughout the Stable Diffusion v1.5 experiments, the VAE, text encoder, tokenizer, and prompt-conditioning pipeline are inherited directly from the pretrained backbone.
We do not modify the conditioning architecture or introduce prompt engineering specifically designed to amplify memorization.
Across different COCO mixing ratios, the backbone architecture, latent consistency distillation setup, solver choice for teacher target construction, and 4-step evaluation protocol are all kept fixed.
This ensures that the main source of variation is the memorization pressure induced during teacher fine-tuning, followed by the same consistency distillation procedure.

\subsection{Experimental Setup in Section~\ref{sec:theory}}
\label{theory_exp}
\subsubsection{Ridge Regularization in Non-Isotropic Consistency Distillation}
\label{subsec:ridge_ucd}

Under consistency distillation, the local probability-flow ODE induces a
non-isotropic response operator of the form
\begin{equation}
\label{eq:A_operator}
\mathbf{A}
=
\mathbf{S}
-
\mu_1^2\,
\mathbf{S}\,(\mathbf{U}+\gamma \mathbf{I})^{-1}\,\mathbf{S},
\end{equation}
where \(\mathbf{U} \in \mathbb{R}^{p\times p}\) denotes the teacher curvature matrix,
\(\mathbf{S} = \frac{1}{d}\mathbf{W}\mathbf{W}^\top\) is the random feature covariance,
and \(\mu_1 = \mathbb{E}[\sigma'(Z)]\).
The inverse operator \((\mathbf{U}+\gamma \mathbf{I})^{-1}\) arises unavoidably from the PF-ODE
linearization and governs how teacher curvature directions are transferred
to the student.

Empirically and theoretically, the spectrum of the teacher curvature \(\mathbf{U}\)
is highly ill-conditioned: a large fraction of its eigenvalues concentrate
near zero, corresponding to memorization-dominated directions, while a
small number of large eigenvalues correspond to generalization-relevant modes.
Denoting the eigendecomposition \(\mathbf{U} = \mathbf{V} \boldsymbol{\Lambda} \mathbf{V}^\top\),
the inverse curvature weights each mode by \((\lambda_k+\gamma)^{-1}\).
Without ridge regularization (\(\gamma=0\)), this induces an extreme amplification
of small-eigenvalue directions,
\begin{equation}
(\mathbf{U})^{-1}
=
\sum_{k=1}^p \frac{1}{\lambda_k} \mathbf{v}_k \mathbf{v}_k^\top,
\end{equation}
causing memorization subspaces to dominate the response even when the input
direction itself lies in a generalization-relevant region.
This phenomenon is not a benign numerical artifact.
In the non-isotropic response energy
\begin{equation}
\label{eq:bi_def}
b_i
=
(\mathbf{S}\mathbf{u}_i)^\top (\mathbf{U}+\gamma \mathbf{I})^{-1} (\mathbf{S}\mathbf{u}_i),
\end{equation}
where \(\mathbf{u}_i\) is an eigenvector of \(\mathbf{U}\),
the contribution from memorization directions can overwhelm that from
generalization directions purely due to inverse spectral weighting.
As a result, response ratios
\(
r_i = \mu_1^2 b_i / (a_i + \varepsilon)
\)
become large even for modes associated with large \(\lambda(\mathbf{U})\),
leading to spurious over-subtraction signals.

To make this effect explicit, decompose
\(
\mathbf{S}\mathbf{u}_i = \sum_{k=1}^p y_{k i} \mathbf{v}_k
\)
in the eigenbasis of \(\mathbf{U}\).
Then
\begin{equation}
b_i
=
\sum_{k=1}^p \frac{y_{k i}^2}{\lambda_k+\gamma}.
\end{equation}
Let \(\mathcal{M}\) denote the memorization subspace, defined by
small eigenvalues of \(\mathbf{U}\).
We define the inverse-weighted memorization leakage as
\begin{equation}
\label{eq:fracBmem}
\mathrm{fracBmem}_i
=
\frac{
\sum_{k\in\mathcal{M}} \frac{y_{k i}^2}{\lambda_k+\gamma}
}{
\sum_{k=1}^p \frac{y_{k i}^2}{\lambda_k+\gamma}
}.
\end{equation}
This quantity directly measures how much of the PF-ODE response energy
of mode \(i\) originates from memorization directions under the inverse
curvature metric.
Without ridge regularization, \(\mathrm{fracBmem}_i \approx 1\) even for
generalization modes, indicating severe cross-subspace leakage.
This behavior invalidates a naive interpretation of the non-isotropic
response as purely suppressive or amplifying.

Introducing a ridge term \(\gamma > 0\) modifies the inverse curvature to
\((\mathbf{U}+\gamma \mathbf{I})^{-1}\), which has two principled effects:
\begin{enumerate}[leftmargin=*,nosep]
\item It bounds the maximum amplification of small-eigenvalue directions,
preventing memorization modes from dominating the response.
\item It restores a meaningful separation between memorization and
generalization subspaces by suppressing inverse-weighted leakage.
\end{enumerate}
Importantly, ridge regularization does not alter the qualitative structure
of the PF-ODE operator; it only controls the conditioning of the inverse
curvature, which is unavoidable in consistency distillation.

\begin{figure}[t]
  \centering
  \includegraphics[width=0.45\linewidth]{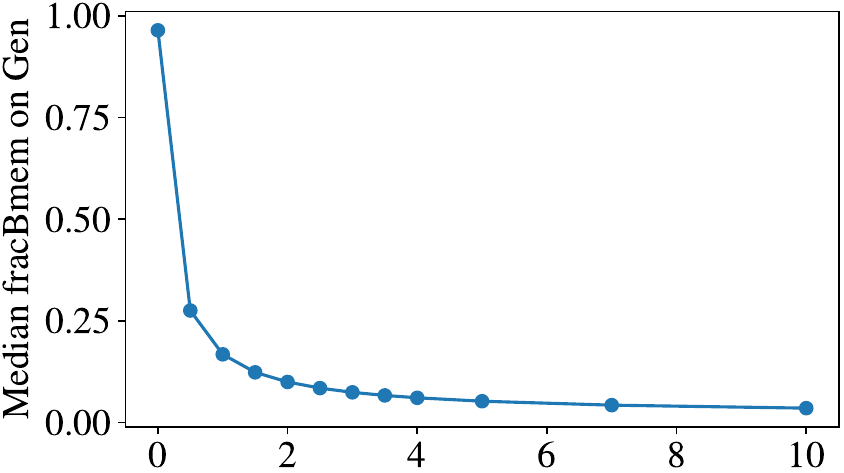}
  \caption{\textbf{Gen-to-Mem leakage after inverse-curvature weighting.}
  Median \(\mathrm{fracBmem}\) on Gen (Eq.~\eqref{eq:fracBmem}) versus ridge \(\gamma\).
  We choose \(\gamma^\star\) by the minimal-sufficient rule in Eq.~\eqref{eq:gamma_star_def}
  with tolerance \(\tau=0.1\).}
  \label{fig:ridge_sweep_fracBmem}
\end{figure}

Rather than selecting \(\gamma\) to enforce a particular sign of the response,
we adopt a geometrically meaningful criterion:
\begin{equation}
\mathrm{median}_{i \in \mathrm{GEN}}
\bigl[
\mathrm{fracBmem}_i
\bigr]
\;\le\; \tau,
\end{equation}
where \(\mathrm{GEN}\) denotes the generalization subspace and \(\tau\)
is a small constant (e.g., \(0.1\) or \(0.2\)).
This criterion ensures that, for generalization modes,
the PF-ODE response is not dominated by memorization directions.

Fig.~\ref{fig:ridge_sweep_fracBmem} reports the ridge sweep of the proposed
leakage statistic \(\mathrm{fracBmem}\) (Eq.~\eqref{eq:fracBmem}) on the
generalization subspace:
\(
\gamma \mapsto
\mathrm{median}_{i\in \mathrm{GEN}}\big[\mathrm{fracBmem}_i(\gamma)\big].
\)
As \(\gamma\) increases, the inverse-curvature weighting
\((\lambda+\gamma)^{-1}\) caps the amplification of small-eigenvalue directions,
yielding a monotone reduction of memorization leakage into GEN after applying
\((\mathbf{U}+\gamma \mathbf{I})^{-1}\).

To make the ridge choice reproducible, we select the \emph{minimal sufficient}
regularization level \(\gamma^\star\) that enforces a target leakage tolerance
\(\tau\):
\begin{equation}
\label{eq:gamma_star_def}
\gamma^\star
=
\min\Bigl\{\gamma>0:\;
\mathrm{median}_{i\in \mathrm{GEN}}\big[\mathrm{fracBmem}_i(\gamma)\big]
\le \tau
\Bigr\}.
\end{equation}
In our experiments, setting \(\tau=0.1\) yields \(\gamma^\star \approx 2.0\),
since the GEN median leakage drops from \(0.123\) at \(\gamma=1.5\) to
\(0.0997\) at \(\gamma=2.0\), while larger ridge values produce diminishing
returns in leakage reduction.
This choice controls cross-subspace contamination induced by the inverse
curvature metric, without tuning \(\gamma\) to force a particular sign pattern
of the response.

Fig.~\ref{fig:ridge_sweep_fracBmem} further shows that the unregularized operator
(\(\gamma=0\)) yields near-total GEN-to-MEM leakage after
\((\mathbf{U}+\gamma \mathbf{I})^{-1}\).
Increasing \(\gamma\) rapidly suppresses this effect; beyond \(\gamma \approx 2\),
the leakage curve flattens, indicating that \(\gamma^\star\) captures the main
conditioning benefit while avoiding excessive attenuation of the overall
non-isotropic structure.

\subsubsection{Numerical Setup for Computing the Spectrum of $\mathbf{U}_{\mathrm{cd}}$}

We compute the empirical spectrum of the one-step consistency distillation
curvature operator under the Gaussian-equivalent RFNN model described in
Section~\ref{sec:theory}.
All reported metrics are evaluated for the full one-step curvature
$\mathbf{U}_{\mathrm{cd}}$, which admits the small-step approximation
\[
\mathbf{U}_{\mathrm{cd}}
=
\Delta t^2\, a_1(t')^2
\Bigl(
\mathbf{S}
-
\mu_1(t')^2\, \mathbf{S}\mathbf{U}^{-1}\mathbf{S}
\Bigr)
+
\beta(t',\Delta t)\,\mathbf{I}.
\]

All experiments are conducted with ambient dimension fixed to $d=100$.
The number of random features and effective training samples scale linearly with
$d$ as $p=\psi_p d$ and $n=\psi_n d$, where we use
$\psi_p=32$ and $\psi_n=4$ throughout, corresponding to $p=3200$ and $n=400$.
Curvature metrics are evaluated at a fixed diffusion time $t'=10^{-2}$ under the
OU forward process, and the consistency distillation update is approximated using
a single Euler step with step size $\Delta t=10^{-3}$.
The RFNN uses the $\tanh(\cdot)$ activation function, and the data covariance is
assumed isotropic, $\Sigma=\mathbf{I}_d$, corresponding to
$\rho_\Sigma(\lambda)=\delta(\lambda-1)$.

Teacher-dependent constants $(a_t,b_t,v_t,s_t^2)$ are estimated via Monte Carlo
sampling under the OU forward process using $2\times10^5$ samples.
The Gaussian-equivalent curvature matrices are then constructed as
\[
\mathbf{S} = \frac{1}{d}\mathbf{W}\mathbf{W}^\top,
\qquad
\mathbf{U} = \frac{1}{n}\mathbf{G}\mathbf{G}^\top + b_t^2 \mathbf{S} + s_t^2 \mathbf{I},
\]
where $\mathbf{W}\in\mathbb{R}^{p\times d}$ and the auxiliary random matrices used
to form $\mathbf{G}$ have i.i.d.\ standard Gaussian entries.
To avoid trace-based closure approximations, PF-ODE constants are estimated
directly by sampling $x\sim p_{t'}$ and computing
$\eta=\mathbb{E}[x^\top s_\phi(x)]/d$ and
$\upsilon=\mathbb{E}[\|s_\phi(x)\|_2^2]/d$ using $5\times10^4$ Monte Carlo samples.
These estimates define $\gamma=1+\eta$ and $\kappa^2=1+2\eta+\upsilon$, which are
used to compute the closed-form coefficients $a_1(t')$ and $a_0(t')$ appearing in
$\mathbf{U}_{\mathrm{cd}}$.

The eigenvalue spectrum of $\mathbf{U}_{\mathrm{cd}}$ is computed via dense
eigendecomposition.
To isolate the continuous spectral component, eigenvalues below
$\varepsilon_{\mathrm{atom}}=10^{-50}$ are discarded, and all reported histograms
and summary statistics are computed over the remaining non-zero spectral support.

\subsubsection{Numerical setup for RFNN-based non-isotropic CD diagnostics}
\label{appendix_rfnn_noniso_numeric_setup}

We describe the numerical setup used to evaluate RFNN-based diagnostics for the
non-isotropic one-step consistency distillation operator.

All experiments are conducted under isotropic data covariance
$\rho_{\Sigma}(\lambda)=\delta(\lambda-1)$ with input dimension fixed to $d=100$.
The random feature and sample dimensions scale linearly with $d$ as
$p=\psi_p d$ and $n=\psi_n d$, where we use $\psi_p=32$ and $\psi_n=4$ throughout,
corresponding to $p=3200$ and $n=400$.
We consider the OU forward process at a fixed diffusion time $t'=0.01$ and use a
single Euler step of size $\Delta t=10^{-3}$ in all reported experiments.

Random features are constructed by sampling
$\mathbf{W}\in\mathbb{R}^{p\times d}$ with i.i.d.\ $\mathcal{N}(0,1)$ entries and
setting $\mathbf{B}=\mathbf{W}/\sqrt d$, yielding the metric
$\mathbf{S}=\mathbf{B}\mathbf{B}^\top$.
Teacher-dependent constants $(a_t,b_t,v_t,s_t^2)$ and
$\mu_1=\mathbb{E}[\sigma'(Z)]$ for $\sigma=\tanh$ are estimated via Monte Carlo
sampling using $5\times10^5$ samples.
The Gaussian-equivalent teacher curvature is constructed as
\[
\mathbf{U}
=
\frac{1}{n}\mathbf{G}\mathbf{G}^\top
+
b_t^2 \mathbf{S}
+
s_t^2 \mathbf{I}_p,
\qquad
\mathbf{G}
=
e^{-t'} a_t\,(\mathbf{B}\mathbf{X}')
+
v_t\,\boldsymbol{\Omega},
\]
where $\mathbf{X}'\in\mathbb{R}^{d\times n}$ and
$\boldsymbol{\Omega}\in\mathbb{R}^{p\times n}$ have i.i.d.\ standard Gaussian
entries.
The non-isotropic channel operator is defined as
\[
\mathbf{A}
=
\mathbf{S}
-
\mu_1^2\,\mathbf{S}(\mathbf{U}+\gamma\mathbf{I})^{-1}\mathbf{S},
\]
with ridge parameter $\gamma=2$ used throughout to ensure numerical stability.

Diagnostics are evaluated along the eigenmodes
$\{\mathbf{u}_i\}_{i=1}^p$ of $\mathbf{U}$.
Memorization- and generalization-associated modes are separated using a fixed
spectral threshold $\lambda_{\mathrm{th}}=2$, with
$\lambda_i(\mathbf{U})<\lambda_{\mathrm{th}}$ classified as Mem and
$\lambda_i(\mathbf{U})\ge\lambda_{\mathrm{th}}$ as Gen.
For each mode we compute
\[
a_i=\mathbf{u}_i^\top \mathbf{S}\mathbf{u}_i,\quad
b_i=\mathbf{u}_i^\top \mathbf{S}(\mathbf{U}+\gamma\mathbf{I})^{-1}\mathbf{S}\mathbf{u}_i,\quad
\alpha_i=a_i-\mu_1^2 b_i.
\]

\section{Proofs}
\label{proof_all}

\subsection{Assumptions}
\label{assumption_appendix}

\begin{assumption}[Gaussian-equivalent random features]
\label{ass:gaussian_equivalence}

$\mathbf{W}\in\mathbb R^{p\times d}$ have i.i.d.\ $\mathcal N(0,1)$ entries and define
$
g(\mathbf{x})=\frac{\mathbf{W}\mathbf{x}}{\sqrt d}$,
$h(\mathbf{x})=\sigma(g(\mathbf{x}))$.
For a given perturbation $\boldsymbol{\delta x}\in\mathbb R^d$, set
$
\Delta g=g(\mathbf{x}+\boldsymbol{\delta x})-g(\mathbf{x})
=\frac{\mathbf{W}\boldsymbol{\delta x}}{\sqrt d}$, $\Delta h=h(\mathbf{x})-h(\mathbf{x}+\boldsymbol{\delta x})$ and
$\mathbf{w}_i^\top$ denote the $i$-th row of $\mathbf{W}$.
Conditionally on $(\mathbf{x},\boldsymbol{\delta x})$ and with respect to the randomness of
$\mathbf{w}_i\sim\mathcal N(0,I_d)$, the coordinate pair
$
(g_i,\Delta g_i)
=
\left(
\frac{\mathbf{w}_i^\top \mathbf{x}}{\sqrt d},
\frac{\mathbf{w}_i^\top \boldsymbol{\delta x}}{\sqrt d}
\right)$
is centered jointly Gaussian with
$\Gamma_d^2 = \mathrm{Var}(g_i\mid   \mathbf{x}) = \frac{\|\mathbf{x}\|_2^2}{d}$,
$\Delta_d^2 = \mathrm{Var}(\Delta g_i\mid  \mathbf{x},\boldsymbol{\delta x}) = \frac{\|\boldsymbol{\delta x}\|_2^2}{d}$,
$c_d = \mathrm{Cov}(g_i,\Delta g_i\mid  \mathbf{x},\boldsymbol{\delta x}) = \frac{\mathbf{x}^\top \boldsymbol{\delta x}}{d}$.
Moreover, the pairs $\{(g_i,\Delta g_i)\}_{i=1}^p$ are i.i.d.\ across $i$.
\end{assumption}

\begin{assumption}[Activation moment and smoothness conditions]
\label{ass:activation_moment}

The activation $\sigma:\mathbb R\to\mathbb R$ is measurable and satisfies
$
\mathbb E_\zeta \left[
(\sigma(g_i)-\sigma(g_i+\Delta g_i))^2
\;\middle|\;
\mathbf{x},\boldsymbol{\delta x}
\right]
<\infty $.
In addition, $\sigma$ is almost everywhere differentiable and
$
\mathbb E_{G\sim\mathcal N(0,1)} \big[\sigma'(G)^2\big]<\infty$.
For sharper control of higher-order terms, we assume
$\sigma\in C^2$ with
$\mathbb E_{G\sim\mathcal N(0,1)} \big[\sigma''(G)^2\big]<\infty$.
\end{assumption}

\begin{assumption}[Small-noise one-step regime]
\label{ass:small_noise}

As $d\to\infty$ and $\Delta t\to 0$, the perturbation satisfies
$
\frac{\|\boldsymbol{\delta x}\|_2^2}{d}\longrightarrow 0$,
$\Gamma_d^2=\frac{\|\mathbf{x}\|_2^2}{d}\longrightarrow \Gamma^2\in(0,\infty)$,
so that
$
\Delta_d^2=\mathrm{Var}(\Delta g_i)=\frac{\|\boldsymbol{\delta x}\|_2^2}{d}\longrightarrow 0$.
In the isotropic setting $\boldsymbol{\Sigma}=I_d$ and for the one-step
OU probability flow update,
we further assume the existence of deterministic limits
$
\eta(t')=\lim_{d\to\infty}\frac{\mathbf{x}^\top \mathbf{s}_\phi(\mathbf{x},t')}{d}$ and
$\upsilon(t')=\lim_{d\to\infty}\frac{\|\mathbf{s}_\phi(\mathbf{x},t')\|_2^2}{d}$ holding in probability.
Consequently, we have
\[
\Gamma_d^2\to 1,
c_d=\frac{\mathbf{x}^\top\boldsymbol{\delta x}}{d}\to -\Delta t\,\gamma(t'),
\Delta_d^2\to \Delta t^2\,\kappa(t')^2,
\]
where
$
\gamma(t')=1+\eta(t'),
\kappa(t')^2=1+2\eta(t')+\upsilon(t')$.
\end{assumption}

\subsection{Proof of Lemma~\ref{lem:orth_decomp_delta_h}}
\label{lemma_proof}
\label{subsec:delta_h_generic}

\begin{proof}
Fix $(\mathbf{x},\boldsymbol{\delta x})\in\mathbb R^d\times\mathbb R^d$.
Let $\mathbf{w}\sim\mathcal N(0,\mathbf{I}_d)$ and define the single-coordinate Gaussian pair
\[
G=\frac{\mathbf{w}^\top \mathbf{x}}{\sqrt d},
\qquad
\Delta G=\frac{\mathbf{w}^\top \boldsymbol{\delta x}}{\sqrt d}.
\]
By Assumption~\ref{ass:gaussian_equivalence}, conditionally on $(\mathbf{x},\boldsymbol{\delta x})$,
$(G,\Delta G)$ is centered jointly Gaussian with
\[
\mathrm{Var}(G)=\Gamma_d^2,\qquad
\mathrm{Var}(\Delta G)=\Delta_d^2,\qquad
\mathrm{Cov}(G,\Delta G)=c_d.
\]
Let $Y=\sigma(G)-\sigma(G+\Delta G)$ and $X=\Delta G$.
By Assumption~\ref{ass:activation_moment}, $Y\in L^2$ and $X\in L^2$.
Assume $\Delta_d^2>0$ so that $\mathbb E[X^2\mid \mathbf{x},\boldsymbol{\delta x}]=\Delta_d^2>0$.

Define the projection coefficient
\[
a_1(\mathbf{x},\boldsymbol{\delta x})=\frac{\mathbb E[Y X\mid \mathbf{x},\boldsymbol{\delta x}]}{\mathbb E[X^2\mid \mathbf{x},\boldsymbol{\delta x}]},
\qquad
R=Y-a_1(\mathbf{x},\boldsymbol{\delta x})\,X.
\]
Then $a_1(\mathbf{x},\boldsymbol{\delta x})$ is the $L^2$-projection coefficient of $Y$ onto $\mathrm{span}\{X\}$, hence
\begin{equation}
\mathbb E[RX\mid \mathbf{x},\boldsymbol{\delta x}]=0.
\label{eq:R_orth}
\end{equation}
Moreover,
\begin{align}
\mathbb E[R^2\mid \mathbf{x},\boldsymbol{\delta x}]
&=
\mathbb E\big[(Y-a_1X)^2\mid \mathbf{x},\boldsymbol{\delta x}\big]\notag\\
&=
\mathbb E[Y^2\mid \mathbf{x},\boldsymbol{\delta x}]
-a_1(\mathbf{x},\boldsymbol{\delta x})^2\,\mathbb E[X^2\mid \mathbf{x},\boldsymbol{\delta x}]\notag\\
&=
a_0(\mathbf{x},\boldsymbol{\delta x})\,\mathbb E[X^2\mid \mathbf{x},\boldsymbol{\delta x}],
\label{eq:R2_def}
\end{align}
where $a_0(\mathbf{x},\boldsymbol{\delta x})$ is defined in Lemma~\ref{lem:orth_decomp_delta_h}.
Now lift from a single coordinate to the full vector.
For each $i\in\{1,\dots,p\}$, let $(g_i,\Delta g_i)$ be i.i.d.\ copies of $(G,\Delta G)$
under the conditional law induced by $\mathbf{w}_i\sim\mathcal N(0,\mathbf{I}_d)$, and set
\[
\Delta h_i=\sigma(g_i)-\sigma(g_i+\Delta g_i),
\qquad
R_i=\Delta h_i-a_1(\mathbf{x},\boldsymbol{\delta x})\,\Delta g_i.
\]
Writing $\Delta \mathbf{h}=a_1\Delta \mathbf{g}+ \mathbf{R}$ and expanding second moments gives
\[
\mathbb E[\Delta \mathbf{h}\Delta \mathbf{h}^\top   \mid  \mathbf{x},\boldsymbol{\delta x}]
 = 
a_1^2\,\mathbb E[\Delta \mathbf{g}\Delta \mathbf{g}^\top   \mid  \mathbf{x},\boldsymbol{\delta x}]
 + 
a_1\,\mathbb E[\Delta \mathbf{g}\mathbf{R}^\top   \mid  \mathbf{x},\boldsymbol{\delta x}]
 + 
a_1\,\mathbb E[\mathbf{R}\Delta \mathbf{g}^\top   \mid  \mathbf{x},\boldsymbol{\delta x}]
 + 
\mathbb E[\mathbf{R}\mathbf{R}^\top     \mid  \mathbf{x},\boldsymbol{\delta x}].
\]
By (\ref{eq:R_orth}) applied coordinate-wise and independence across $i$,
the cross terms vanish:
$\mathbb E[\Delta \mathbf{g}\mathbf{R}^\top   \mid  \mathbf{x},\boldsymbol{\delta x}]=0$ and $\mathbb E[\mathbf{R}\Delta \mathbf{g}^\top  \mid  \mathbf{x},\boldsymbol{\delta x}]=0$.
Furthermore, since the coordinates are i.i.d.\ and $R_i$ has conditional second moment
$\mathbb E[R_i^2\mid \mathbf{x},\boldsymbol{\delta x}]=\mathbb E[R^2\mid \mathbf{x},\boldsymbol{\delta x}]$,
we have $\mathbb E[\mathbf{R}\mathbf{R}^\top   \mid  \mathbf{x},\boldsymbol{\delta x}]=\mathbb E[R^2\mid \mathbf{x},\boldsymbol{\delta x}]\;\mathbf{I}_p$.
Using (\ref{eq:R2_def}) and $\mathbb E[X^2\mid  \mathbf{x},\boldsymbol{\delta x}]=\mathbb E[\Delta g_i^2\mid \mathbf{x},\boldsymbol{\delta x}]$
yields the exact decomposition (\ref{eq:delta_h_second_moment_decomp}).

Assume now the regime of Assumption~\ref{ass:small_noise} with $\boldsymbol{\Sigma}=\mathbf{I}_d$ and the one-step PF-ODE update (\ref{eq:pf_ode_ou}).
Let $G_d=G/\Gamma_d$ so that $G_d\sim\mathcal N(0,1)$.
Define
$
Z=\frac{\Delta G}{\Delta t}$,
we have $\Delta G=\Delta tZ$.
By Assumption~\ref{ass:small_noise}, the joint Gaussian parameters satisfy
\[
\Gamma_d^2\to 1,
\qquad
\mathrm{Var}(Z)=\frac{\Delta_d^2}{\Delta t^2}\to \kappa(t')^2,
\qquad
\mathrm{Cov}(G,Z)=\frac{c_d}{\Delta t}\to \gamma(t').
\]
Hence $(G,Z)$ converges in distribution to a centered jointly Gaussian pair with
\[
G\sim\mathcal N(0,1),
\qquad
\mathbb E[Z^2]=\kappa(t')^2,
\qquad
\mathbb E[GZ]=\gamma(t').
\]

We next compute the leading-order limits of $a_1(\mathbf{x},\boldsymbol{\delta x})$ and $a_0(\mathbf{x},\boldsymbol{\delta x})$.
Write $X=\Delta G$ and $Y=\sigma(G)-\sigma(G+\Delta G)$ as above.
Using the mean-value form of Taylor's theorem, for each realization there exists
$\theta\in(0,1)$ such that
\[
\sigma(G+\Delta G)=\sigma(G)+\sigma'(G)\Delta G+\frac12\,\sigma''(G+\theta\Delta G)\,(\Delta G)^2.
\]
Thus
\begin{equation}
Y
=
-\sigma'(G)\Delta G-\rho,
\qquad
\rho=\frac12\,\sigma''(G+\theta\Delta G)\,(\Delta G)^2.
\label{eq:Y_expand_pf}
\end{equation}
Under the smoothness conditions in Assumption~\ref{ass:activation_moment} and
since $\mathbb E[(\Delta G)^4]=O(\Delta_d^4)=O(\Delta t^4)$, we have
$\mathbb E[\rho^2]=O(\Delta t^4)$ and hence
$\mathbb E|\rho\,\Delta G|=o(\Delta t^2)$.

\noindent\emph{Limit of $a_1$.}
Using (\ref{eq:Y_expand_pf}) and $X=\Delta G$,
\begin{align}
\mathbb E[YX\mid \mathbf{x},\boldsymbol{\delta x}]
&=
\mathbb E\big[Y\Delta G\mid \mathbf{x},\boldsymbol{\delta x}\big]\notag\\
&=
-\mathbb E\big[\sigma'(G)(\Delta G)^2\mid \mathbf{x},\boldsymbol{\delta x}\big]
-\mathbb E\big[\rho\,\Delta G\mid \mathbf{x},\boldsymbol{\delta x}\big]\notag\\
&=
-\Delta t^2\,\mathbb E\big[\sigma'(G)Z^2\mid \mathbf{x},\boldsymbol{\delta x}\big]
+o(\Delta t^2),
\label{eq:YX_limit_pf}
\end{align}
while
\[
\mathbb E[X^2\mid \mathbf{x},\boldsymbol{\delta x}]
=
\mathbb E[(\Delta G)^2\mid \mathbf{x},\boldsymbol{\delta x}]
=
\Delta t^2\,\mathbb E[Z^2\mid \mathbf{x},\boldsymbol{\delta x}]
=
\Delta t^2\,\kappa(t')^2+o(\Delta t^2).
\]
Therefore,
\begin{equation}
a_1(\mathbf{x},\boldsymbol{\delta x})
=
\frac{\mathbb E[YX\mid \mathbf{x},\boldsymbol{\delta x}]}{\mathbb E[X^2\mid \mathbf{x},\boldsymbol{\delta x}]}
=
-\frac{\mathbb E[\sigma'(G)Z^2]}{\kappa(t')^2}+o(1),
\label{eq:a1_reduce_pf}
\end{equation}
where expectations on the right-hand side are taken under the limiting
joint Gaussian law of $(G,Z)$.

Since $(G,Z)$ is jointly Gaussian with $\mathbb E[GZ]=\gamma(t')$ and $\mathbb E[Z^2]=\kappa(t')^2$,
we have
\[
\mathbb E[Z^2\mid G]
=
\big(\mathbb E[Z\mid G]\big)^2+\mathrm{Var}(Z\mid G)
=
\gamma(t')^2\,G^2+\big(\kappa(t')^2-\gamma(t')^2\big).
\]
Hence
\begin{align}
\mathbb E[\sigma'(G)Z^2]
=
\mathbb E\big[\sigma'(G)\,\mathbb E[Z^2\mid G]\big]
=
\gamma(t')^2\,\mathbb E\big[\sigma'(G)G^2\big]
+\big(\kappa(t')^2-\gamma(t')^2\big)\,\mathbb E\big[\sigma'(G)\big],
\label{eq:EZ2_sigma_pf}
\end{align}
with $G\sim\mathcal N(0,1)$.
Combining (\ref{eq:a1_reduce_pf}) and (\ref{eq:EZ2_sigma_pf}) yields
\begin{align}
a_1(t')
=&
-\frac{
\gamma(t')^2\mathbb E_{G} \big[\sigma'(G)\,G^2\big]
+
\big(\kappa(t')^2-\gamma(t')^2\big)
\mathbb E_{G} \big[\sigma'(G)\big]
}{
\kappa(t')^2
}.
\label{eq:a1}
\end{align}


\noindent\emph{Limit of $a_0$.}
Similarly, using (\ref{eq:Y_expand_pf}),
\begin{align}
\mathbb E[Y^2\mid \mathbf{x},\boldsymbol{\delta x}]
&=
\mathbb E\big[\sigma'(G)^2(\Delta G)^2\mid \mathbf{x},\boldsymbol{\delta x}\big]
+2\,\mathbb E\big[\sigma'(G)\Delta G\,\rho\mid \mathbf{x},\boldsymbol{\delta x}\big]
+\mathbb E[\rho^2\mid \mathbf{x},\boldsymbol{\delta x}]\notag\\
&=
\Delta t^2\,\mathbb E\big[\sigma'(G)^2 Z^2\mid \mathbf{x},\boldsymbol{\delta x}\big]
+o(\Delta t^2),
\label{eq:Y2_limit_pf}
\end{align}
where the $o(\Delta t^2)$ term follows from Cauchy--Schwarz together with
$\mathbb E[\rho^2]=O(\Delta t^4)$.
Dividing by $\mathbb E[X^2]=\Delta t^2\kappa(t')^2+o(\Delta t^2)$ gives
\[
\frac{\mathbb E[Y^2]}{\mathbb E[X^2]}
=
\frac{\mathbb E[\sigma'(G)^2 Z^2]}{\kappa(t')^2}+o(1).
\]
Using again $\mathbb E[Z^2\mid G]=\gamma(t')^2G^2+(\kappa(t')^2-\gamma(t')^2)$ yields
\begin{align}
\mathbb E[\sigma'(G)^2 Z^2]
&=
\gamma(t')^2\,\mathbb E\big[\sigma'(G)^2G^2\big]
+\big(\kappa(t')^2-\gamma(t')^2\big)\,\mathbb E\big[\sigma'(G)^2\big],
\label{eq:EZ2_sigmap2_pf}
\end{align}
with $G\sim\mathcal N(0,1)$.
Combining these expressions with the definition
$a_0=\frac{\mathbb E[Y^2]}{\mathbb E[X^2]}-a_1^2$, we have
\begin{align}
a_0(t')
=&
\frac{
\gamma(t')^2\mathbb E_{G} \big[\sigma'(G)^2G^2\big]
+
\big(\kappa(t')^2-\gamma(t')^2\big)
\mathbb E_{G} \big[\sigma'(G)^2\big]
}{
\kappa(t')^2
}
-a_1(t')^2.
\label{eq:a0}
\end{align}
This completes the proof.
\end{proof}

\subsection{Proof of Theorem~\ref{thm:Ucd_structure_pfode_empirical}}
\label{the_proof}
\begin{proof}

Lemma~\ref{lem:orth_decomp_delta_h} gives, for each $(\mathbf{x},\delta \mathbf{x})$,
\[
\mathbb E_{\zeta} \left[\Delta \mathbf{h}\,\Delta \mathbf{h}^\top \mid \mathbf{x},\delta \mathbf{x}\right]
=
a_1(\mathbf{x},\delta \mathbf{x})^2\,
\mathbb E_{\zeta} \left[\Delta \mathbf{g}\,\Delta \mathbf{g}^\top \mid \mathbf{x},\delta \mathbf{x}\right]
+
a_0(\mathbf{x},\delta \mathbf{x})\,
\mathbb E_{\zeta} \left[\Delta g_i^2 \mid \mathbf{x},\delta \mathbf{x}\right]\,
\mathbf{I}_p,
\]
with $\mathbb E_{\zeta}[\Delta g_i^2\mid \mathbf{x},\delta \mathbf{x}]=\|\delta \mathbf{x}\|_2^2/d$.
Taking $\mathbb E_\xi$ and then averaging over $\nu$ yields
\begin{equation}
\mathbf{U}_{\mathrm{cd}}
=
\frac{1}{n}\sum_{\nu=1}^n \mathbb E_\xi \left[
a_1(\mathbf{x},\delta \mathbf{x})^2\,\Delta \mathbf{g}\,\Delta \mathbf{g}^\top
\right]
+
\frac{1}{n}\sum_{\nu=1}^n \mathbb E_\xi \left[
a_0(\mathbf{x},\delta \mathbf{x})\,\frac{\|\delta \mathbf{x}\|_2^2}{d}
\right] \mathbf{I}_p.
\label{eq:Ucd_after_lemma_emp}
\end{equation}

Under the small-noise one-step regime and the concentration hypotheses of
Lemma~\ref{lem:orth_decomp_delta_h}, we may replace
$a_1(\mathbf{x},\delta \mathbf{x})\to a_1(t')$ and
$a_0(\mathbf{x},\delta \mathbf{x})\to a_0(t')$ in (\ref{eq:Ucd_after_lemma_emp})
at leading order, obtaining
\begin{equation}
\mathbf{U}_{\mathrm{cd}}
=
a_1(t')^2\,
\frac{1}{n}\sum_{\nu=1}^n \mathbb E_\xi \left[
\Delta \mathbf{g}\,\Delta \mathbf{g}^\top
\right]
+
\beta(t',\Delta t)\,\mathbf{I}_p
+
o(\Delta t^2),
\label{eq:Ucd_step1_emp}
\end{equation}
where $\beta(t',\Delta t) = \frac{1}{n}\sum_{\nu=1}^n 
\mathbb E_{\xi} \left[\frac{\|\delta \mathbf{x}\|_2^2}{d}\right]$.

Since $\Delta \mathbf{g}=\mathbf{W}\,\delta \mathbf{x}/\sqrt d$, then we have
\[
\Delta \mathbf{g}\,\Delta \mathbf{g}^\top
=
\frac{1}{d}\,\mathbf{W}\,\delta \mathbf{x}\,\delta \mathbf{x}^\top \mathbf{W}^\top,
\quad
\frac{1}{n}\sum_{\nu=1}^n\mathbb E_\xi[\Delta \mathbf{g}\,\Delta \mathbf{g}^\top]
=
\frac{1}{d}\,\mathbf{W}
\Big(
\frac{1}{n}\sum_{\nu=1}^n\mathbb E_\xi[\delta \mathbf{x}\,\delta \mathbf{x}^\top]
\Big)
\mathbf{W}^\top.
\]
Using $\delta \mathbf{x}=-\Delta t(\mathbf{x}+\mathbf{s}_\phi(\mathbf{x},t'))$, we obtain
\begin{equation}
\frac{1}{n}\sum_{\nu=1}^n\mathbb E_\xi[\delta \mathbf{x}\,\delta \mathbf{x}^\top]
=
\Delta t^2\,
\frac{1}{n}\sum_{\nu=1}^n
\mathbb E_\xi \left[(\mathbf{x}+\mathbf{s}_\phi)(\mathbf{x}+\mathbf{s}_\phi)^\top\right].
\label{eq:avg_dx_dxT}
\end{equation}

From (\ref{eq:Aphi_closed_form}) and the deterministic equivalent
(\ref{eq:V_GEP_form}), the converged teacher score can be expressed as
\begin{equation}
\mathbf{s}_\phi(\mathbf{x},t')
=
-\mu_1(t')\,\frac{\mathbf{W}^\top}{\sqrt d}\,\mathbf{U}^{-1}\,\mathbf{h}(\mathbf{x}),
\label{eq:sphi_emp_closed}
\end{equation}
where $\mathbf{U}=(1/n)\sum_{\nu}\mathbb E_\xi[\mathbf{h}(\mathbf{x})\mathbf{h}(\mathbf{x})^\top]$.
Let $\mathbf{B}=\mathbf{W}/\sqrt d$ so that $\mathbf{S}=\mathbf{B}\mathbf{B}^\top$.
Under Gaussian equivalence and $\Sigma=I_d$), the OU marginal of $\mathbf{x}$ is
asymptotically isotropic, and Gaussian integration by parts yields the identity
\begin{equation}
\frac{1}{n}\sum_{\nu=1}^n\mathbb E_\xi \left[\mathbf{x}\,\mathbf{h}(\mathbf{x})^\top\right]
\simeq
\mu_1(t')\,\mathbf{B}^\top,
\label{eq:avg_xhT}
\end{equation}
while by definition of $\mathbf{U}=\frac{1}{n}\sum_{\nu=1}^n\mathbb E_\xi \left[\mathbf{h}(\mathbf{x})\,\mathbf{h}(\mathbf{x})^\top\right]$,
plugging (\ref{eq:sphi_emp_closed}) into the cross and score terms and using
(\ref{eq:avg_xhT}) gives
\begin{align}
\frac{1}{n}\sum_{\nu=1}^n\mathbb E_\xi[\mathbf{x}\mathbf{s}_\phi^\top]
&\simeq
-\mu_1(t')^2\,\mathbf{B}^\top \mathbf{U}^{-1}\mathbf{B},
\quad
\frac{1}{n}\sum_{\nu=1}^n\mathbb E_\xi[\mathbf{s}_\phi \mathbf{x}^\top]
\simeq
-\mu_1(t')^2\,\mathbf{B}^\top \mathbf{U}^{-1}\mathbf{B},
\notag\\
\frac{1}{n}\sum_{\nu=1}^n\mathbb E_\xi[\mathbf{s}_\phi \mathbf{s}_\phi^\top]
&=
\mu_1(t')^2\,\mathbf{B}^\top \mathbf{U}^{-1}
\Big(
\frac{1}{n}\sum_{\nu=1}^n\mathbb E_\xi[\mathbf{h}\mathbf{h}^\top]
\Big)\mathbf{U}^{-1}\mathbf{B}
=
\mu_1(t')^2\,\mathbf{B}^\top \mathbf{U}^{-1}\mathbf{B}.
\label{eq:avg_terms}
\end{align}
Moreover, the isotropic OU marginal implies
\begin{equation}
\frac{1}{n}\sum_{\nu=1}^n\mathbb E_\xi[\mathbf{x}\mathbf{x}^\top]\simeq \mathbf{I}_d.
\label{eq:avg_xxT}
\end{equation}
Combining (\ref{eq:avg_terms}) and (\ref{eq:avg_xxT}) yields:
\begin{equation}
\frac{1}{n}\sum_{\nu=1}^n\mathbb E_\xi \left[(\mathbf{x}+\mathbf{s}_\phi)(\mathbf{x}+\mathbf{s}_\phi)^\top\right]
\simeq
\mathbf{I}_d-\mu_1(t')^2\,\mathbf{B}^\top \mathbf{U}^{-1}\mathbf{B}.
\label{eq:avg_xplusS}
\end{equation}

Substituting (\ref{eq:avg_xplusS}) into (\ref{eq:avg_dx_dxT}) gives
\[
\frac{1}{n}\sum_{\nu=1}^n\mathbb E_\xi[\delta \mathbf{x}\,\delta \mathbf{x}^\top]
\simeq
\Delta t^2\Big(\mathbf{I}_d-\mu_1(t')^2\,\mathbf{B}^\top \mathbf{U}^{-1}\mathbf{B}\Big).
\]
Therefore,
\[
\frac{1}{n}\sum_{\nu=1}^n\mathbb E_\xi[\Delta \mathbf{g}\,\Delta \mathbf{g}^\top]
\simeq
\Delta t^2\,
\frac{1}{d}\,\mathbf{W}\Big(\mathbf{I}_d-\mu_1(t')^2\,\mathbf{B}^\top \mathbf{U}^{-1}\mathbf{B}\Big)\mathbf{W}^\top
=
\Delta t^2\Big(\mathbf{S}-\mu_1(t')^2\,\mathbf{S}\,\mathbf{U}^{-1}\,\mathbf{S}\Big).
\]

Now we begin to analyze the second term in (\ref{eq:Ucd_pfode_structure_emp}).
Using $\delta \mathbf{x} = \Delta t\,t'\,\mathbf{s}_\phi(\mathbf{x},t')$ and
$\mathbf{s}_\phi(\mathbf{x},t')=(1/\sqrt p)\mathbf{A}_\phi \mathbf{h}(\mathbf{x})$, we obtain
\begin{equation}
\frac{\|\delta \mathbf{x}\|_2^2}{d}
=
\Delta t^2 t'^2\,
\frac{1}{d} \frac{1}{p}\,
\mathbf{h}(\mathbf{x})^\top \mathbf{A}_\phi^\top \mathbf{A}_\phi\, \mathbf{h}(\mathbf{x}).
\label{eq:normdx_expand}
\end{equation}
Averaging over $\xi$ and $\nu$ and using
$\mathbf{U}=\frac{1}{n}\sum_{\nu=1}^n \mathbb E_\xi[\mathbf{h}(\mathbf{x})\mathbf{h}(\mathbf{x})^\top]$ yields
\begin{align}
\frac{1}{n}\sum_{\nu=1}^n 
\mathbb E_{\xi} \left[\frac{\|\delta \mathbf{x}\|_2^2}{d}\right]
&=
\frac{1}{n}\sum_{\nu=1}^n 
\mathbb E_{\xi} \left[
\frac{1}{d}
\left\|
\Delta t\,t'\,\frac{1}{\sqrt p}\mathbf{A}_\phi \mathbf{h}(\mathbf{x})
\right\|_2^2
\right] \notag\\
&=
\Delta t^2 t'^2\,
\frac{1}{dp}\,
\frac{1}{n}\sum_{\nu=1}^n 
\mathbb E_{\xi} \left[
\mathbf{h}(\mathbf{x})^\top \mathbf{A}_\phi^\top \mathbf{A}_\phi\, \mathbf{h}(\mathbf{x})
\right].
\label{eq:avg_normdx_step1}
\end{align}
Using the standard quadratic-form identity
$
\mathbb E \left[\mathbf{h}^\top \mathbf{M} \mathbf{h}\right]
=
\mathrm{Tr} \left(\mathbf{M}\,\mathbb E[\mathbf{h}\mathbf{h}^\top]\right),\mathbf{M}\in\mathbb R^{p\times p}$,
we obtain
\begin{align}
\frac{1}{n}\sum_{\nu=1}^n 
\mathbb E_{\xi} \left[
\mathbf{h}(\mathbf{x})^\top \mathbf{A}_\phi^\top \mathbf{A}_\phi\, \mathbf{h}(\mathbf{x})
\right]
&=
\mathrm{Tr} \left(
\mathbf{A}_\phi^\top \mathbf{A}_\phi \cdot
\frac{1}{n}\sum_{\nu=1}^n 
\mathbb E_{\xi} \left[\mathbf{h}(\mathbf{x})\mathbf{h}(\mathbf{x})^\top\right]
\right)
= \mathrm{Tr} \left(\mathbf{A}_\phi^\top \mathbf{A}_\phi\,\mathbf{U}\right).
\label{eq:avg_quad_trace}
\end{align}
Therefore, we finally obtain
\begin{equation}
\frac{1}{n}\sum_{\nu=1}^n 
\mathbb E_{\xi} \left[\frac{\|\delta \mathbf{x}\|_2^2}{d}\right]
=
\Delta t^2 t'^2\,
\frac{1}{dp}\,
\mathrm{Tr} \left(\mathbf{A}_\phi^\top \mathbf{A}_\phi\,\mathbf{U}\right).
\label{eq:avg_normdx_trace_final}
\end{equation}
Invoking Lemma~\ref{lem:teacher_top_layer},
$\mathbf{A}_\phi=-(\sqrt p/\sqrt{\Delta_{t'}})\,\mathbf{V}^\top \mathbf{U}^{-1}$, we have
$
\mathbf{A}_\phi^\top \mathbf{A}_\phi
=
\frac{p}{\Delta_{t'}}\,
\mathbf{U}^{-1} \mathbf{V} \mathbf{V}^\top \mathbf{U}^{-1}$.
Using the Gaussian-equivalent form of $\mathbf{V}$ from Lemma~\ref{lem:teacher_top_layer}, we get
\begin{equation}
\mathbf{V} \mathbf{V}^\top
=
\left(\mu_1(t')\frac{\sqrt{\Delta_{t'}}}{\Gamma_{t'}}\right)^{ 2}
\frac{\mathbf{W}\mathbf{W}^\top}{d}
=
\left(\mu_1(t')\frac{\sqrt{\Delta_{t'}}}{\Gamma_{t'}}\right)^{ 2}
\mathbf{S},
\label{eq:VVt_to_S}
\end{equation}
and therefore
\begin{equation}
\frac{1}{p}\,\mathrm{Tr}(\mathbf{A}_\phi^\top \mathbf{A}_\phi\,\mathbf{U})
=
\frac{1}{\Delta_{t'}}\,
\mathrm{Tr} \left(\mathbf{U}^{-1} \mathbf{V} \mathbf{V}^\top\right)
=
\left(\frac{\mu_1(t')}{\Gamma_{t'}}\right)^{ 2}
\mathrm{Tr} \left(\mathbf{U}^{-1}\mathbf{S}\right).
\label{eq:trace_AphiTAphiU_closed}
\end{equation}
Combining (\ref{eq:normdx_expand})--(\ref{eq:trace_AphiTAphiU_closed}) gives
\begin{equation}
\begin{aligned}
\beta(t',\Delta t)
&=
a_0(t')\,
\Delta t^2 t'^2\,
\left(\frac{\mu_1(t')}{\Gamma_{t'}}\right)^{ 2}
\frac{1}{d}\,\mathrm{Tr} \left(\mathbf{U}^{-1}\mathbf{S}\right)\\
&=a_0(t')\,a_1(t')^2\,
\Delta t^2 t'^2\,
\frac{1}{d}\,\mathrm{Tr} \left(\mathbf{U}^{-1}\mathbf{S}\right).
\label{eq:beta_final}
\end{aligned}
\end{equation}
Putting together the non-isotropic and isotropic contributions,
we finally obtain
\begin{equation}
\mathbf{U}_{\mathrm{cd}}
=
\Delta t^2\,a_1(t')^2
\Big(
\mathbf{S}
-\mu_1(t')^2\,\mathbf{S}\,\mathbf{U}^{-1}\,\mathbf{S}
\Big)
\;+\;
\beta(t',\Delta t)\,\mathbf{I}_p,
\label{eq:Ucd_pfode_structure_emp}
\end{equation}
This completes the proof.

\end{proof}

\section{Additional Results}
\label{Additional_Results}

\subsection{Disentangling Model and Sampler Effects}
\label{Sampler_Effects}

To determine whether the observed reduction in memorization should be attributed to the distilled model itself or to the sampling procedure, we consider two complementary comparisons. The first is reported in the main text: Table~\ref{tab:cifar10_mem_fid} compares the 1-step student with deterministic EDM sampling, which is the evaluation setting most directly aligned with our theoretical analysis of the one-step consistency objective. The results there already show substantial memorization suppression in the 1-step setting, indicating that the effect is not a consequence of multi-step inference.

We further examine this question under matched stochastic sampling. Specifically, we compare two-step consistency models with EDM teachers under stochastic sampling and otherwise identical settings.
Because both models are evaluated with stochastic inference, the resulting memorization gap cannot be explained by differences between ODE and stochastic sampling. As shown in Table~\ref{tab:stochastic_sampling_comparison}, the two-step consistency model consistently exhibits substantially lower memorization than the stochastic EDM teacher under both $l_2$ and SSCD metrics, while also achieving better FID across all settings. This comparison therefore isolates the contribution of the distilled model and provides further evidence that the observed memorization reduction is not induced by the sampler.

\begin{table}[t]
\centering
\caption{\textbf{Comparison between stochastic EDM teachers and two-step consistency models under matched stochastic sampling settings on CIFAR-10}. Since both methods use stochastic sampling, the memorization gap isolates the model effect rather than the sampler effect.}
\label{tab:stochastic_sampling_comparison}
\small
\begin{tabular}{ccccc}
\toprule
Dataset size & Sampler/Model & FID & $l_2$ Mem & SSCD Mem / p95 \\
\midrule
3000 & Teacher stochastic & 16.76 & 14.60\% & 39.08\% / 0.8868 \\
3000 & Student 2-step & 12.95 & 3.21\% & 39.55\% / 0.7792 \\
\midrule
4000 & Teacher stochastic & 18.44 & 16.25\% & 39.92\% / 0.8973 \\
4000 & Student 2-step & 16.57 & 1.76\% & 27.07\% / 0.7579 \\
\midrule
5000 & Teacher stochastic & 21.81 & 12.50\% & 32.24\% / 0.8796 \\
5000 & Student 2-step & 19.52 & 0.60\% & 15.10\% / 0.7064 \\
\midrule
6000 & Teacher stochastic & 22.97 & 13.26\% & 33.12\% / 0.8901 \\
6000 & Student 2-step & 21.23 & 0.48\% & 11.11\% / 0.6783 \\
\bottomrule
\end{tabular}
\end{table}

\subsection{Effect of Model Capacity on Memorization and Distillation}
\label{sec:model_capacity}

We further study how the model capacity affects memorization behavior and how effectively consistency distillation mitigates such effects.
All results in this subsection are obtained using the \emph{two-step} consistency distillation objective, as it yields better generation quality while exhibiting the same qualitative trends as the one-step setting.
All experiments are conducted on a randomly sampled subset of $6000$ training images and trained under an identical optimization schedule, while the total number of training steps varies across configurations according to their respective training setups.
For each architectural configuration, the teacher model used for distillation is selected from the terminal checkpoint of the corresponding training trajectory.

We first vary the model capacity by changing the \emph{base channel width} of the U-Net backbone while keeping the depth fixed.
When the base channel width is reduced from $128$ to $64$, the teacher model exhibits negligible memorization throughout training.
At the terminal training step ($400$k steps), the teacher achieves FID = $24.8$ with a memorization ratio below $0.1\%$.
Applying two-step consistency distillation to this teacher yields a student model with comparable sample quality (FID = $24.58$) and a memorization ratio statistically indistinguishable from zero.

\begin{figure*}[t]
  \centering
  \begin{subfigure}[t]{0.32\textwidth}
    \centering
    \includegraphics[width=\linewidth]{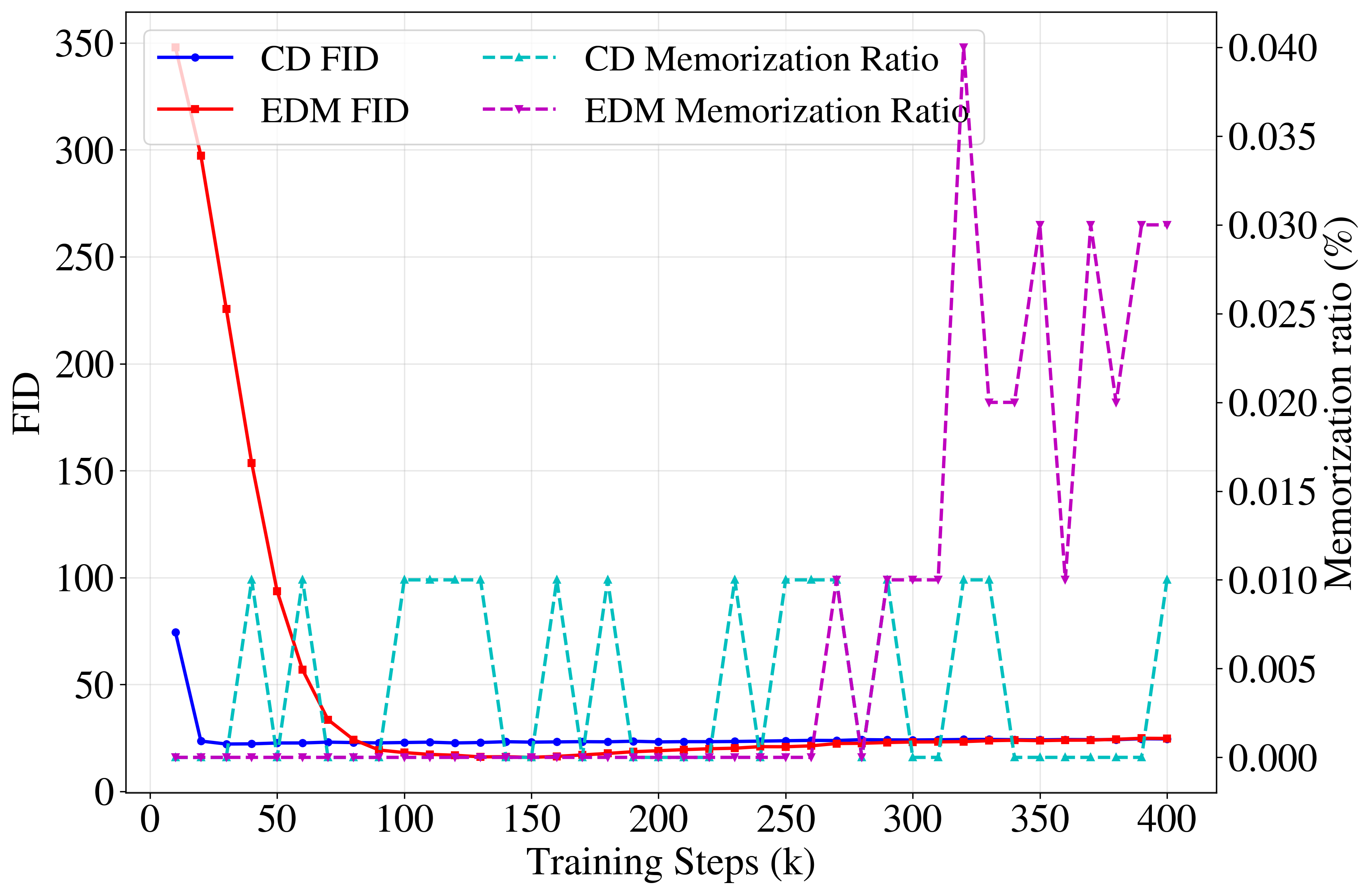}
    \caption{Base channel width decreased from $128$ to $64$.}
    \label{fig:capacity_channel64}
  \end{subfigure}
  \hfill
  \begin{subfigure}[t]{0.32\textwidth}
    \centering
    \includegraphics[width=\linewidth]{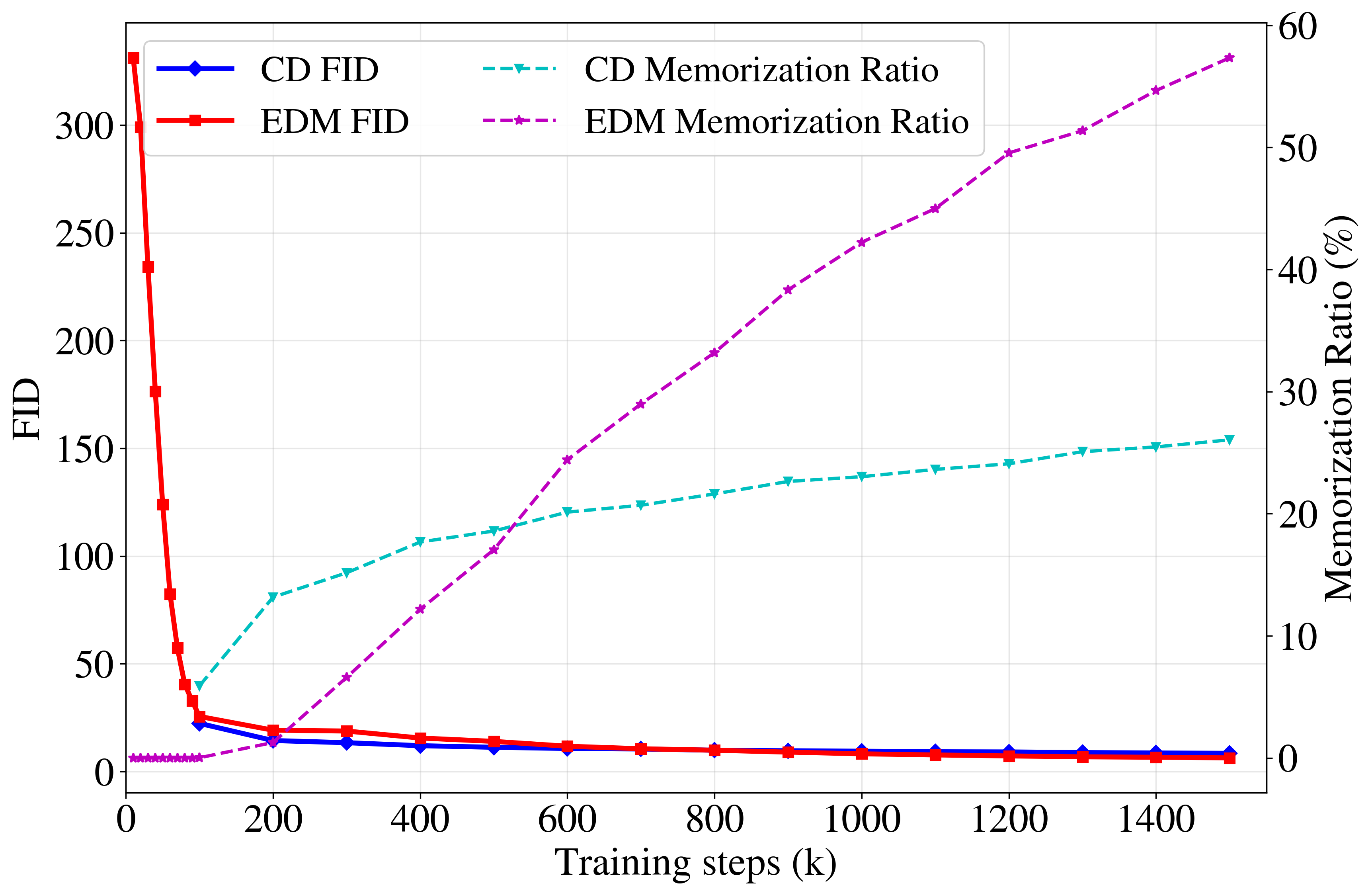}
    \caption{Base channel width increased from $128$ to $192$.}
    \label{fig:capacity_channel192}
  \end{subfigure}
  \hfill
  \begin{subfigure}[t]{0.32\textwidth}
    \centering
    \includegraphics[width=\linewidth]{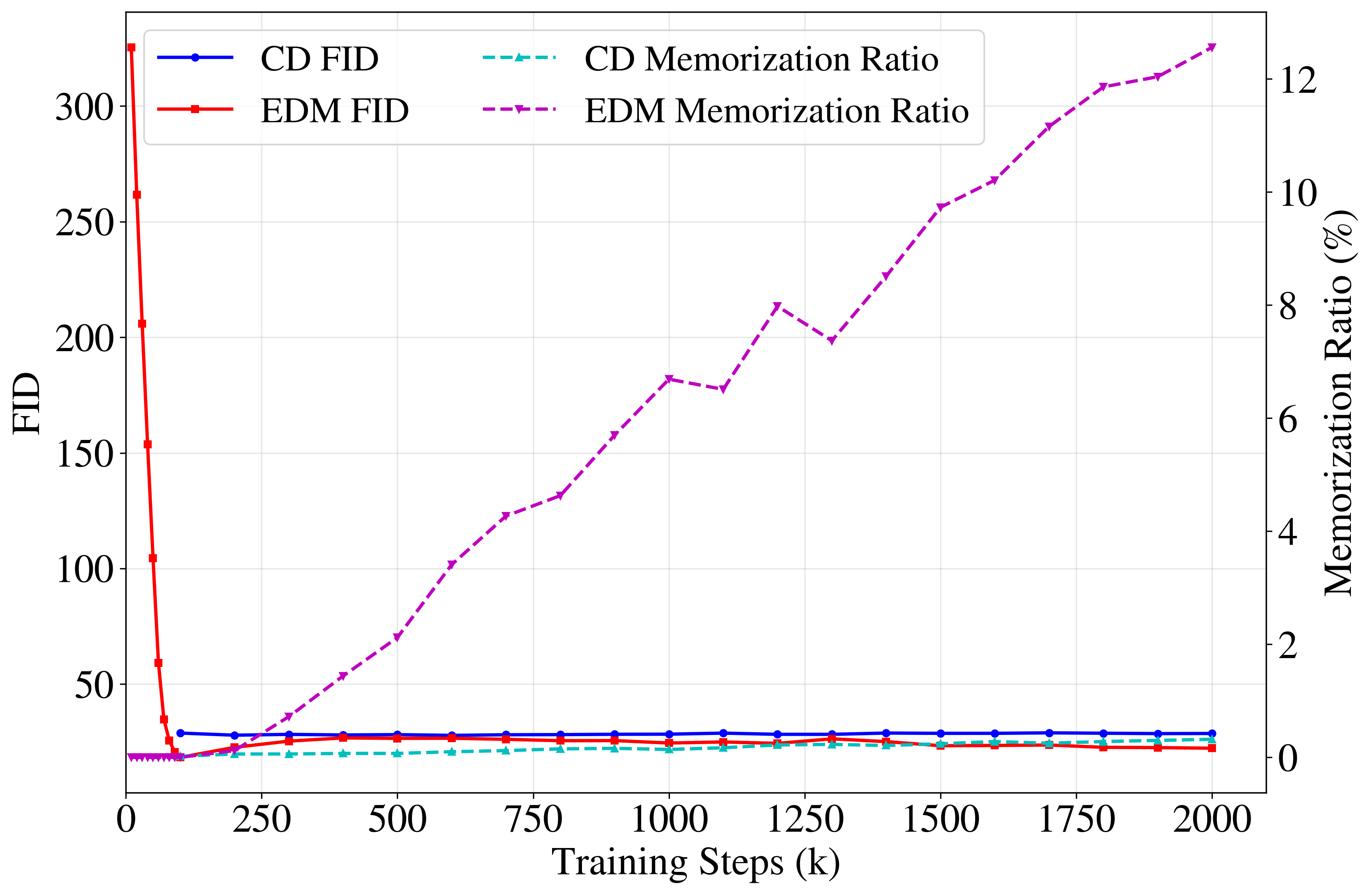}
    \caption{Reduced depth from $4$ residual blocks to $2$.}
    \label{fig:capacity_resblock2}
  \end{subfigure}

  \caption{
  \textbf{Effect of teacher model capacity on memorization and consistency distillation.}
  Each panel reports the evolution of FID (left axis) and memorization ratio (right axis) as a function of training steps for the teacher diffusion model (EDM) and the corresponding consistency-distilled (CD) student.
  Model capacity is varied along three axes: network width (channels) and depth (number of ResNet blocks).
  In all cases, the teacher used for distillation is selected from the terminal training checkpoint of the corresponding run.
  Reducing model capacity—either by decreasing width or depth—suppresses memorization throughout training, and the CD student distilled from such teachers exhibits near-zero memorization.
  Conversely, increasing model capacity leads to earlier and stronger memorization in the teacher, while consistency distillation consistently yields students with substantially reduced memorization at comparable FID.
  }\vspace{-4mm}
  \label{fig:capacity_ablation}
\end{figure*}

In contrast, increasing the base channel width to $192$ substantially alters the training dynamics.
In this high-capacity regime, the teacher model exhibits early onset and steady growth of memorization, reaching a memorization ratio of approximately $57.34\%$ at the terminal checkpoint ($1500$k steps), while achieving FID = $6.34$.
Despite this pronounced memorization in the teacher, the two-step consistency-distilled student reduces the memorization ratio to $26.05\%$, while preserving comparable sample quality with FID = $8.45$.

We further examine architectural depth by reducing the \emph{number of residual blocks per resolution level} from $4$ to $2$, while keeping the base channel width fixed with $128$.
This shallower teacher model exhibits weak memorization, with the terminal memorization ratio of approximately $12.56\%$ and FID = $ 22.31$ ($2000$k steps).
The corresponding two-step consistency-distilled student closely matches this behavior, achieving FID = $28.67$ with a memorization ratio again near zero.
Overall, two-step consistency distillation behaves robustly across model capacities: it preserves non-memorizing behavior when the teacher does not memorize, and substantially suppresses memorization when increased capacity induces it, while maintaining competitive sample quality.

\begin{figure*}[t]
  \centering
  \includegraphics[width=0.9\textwidth]{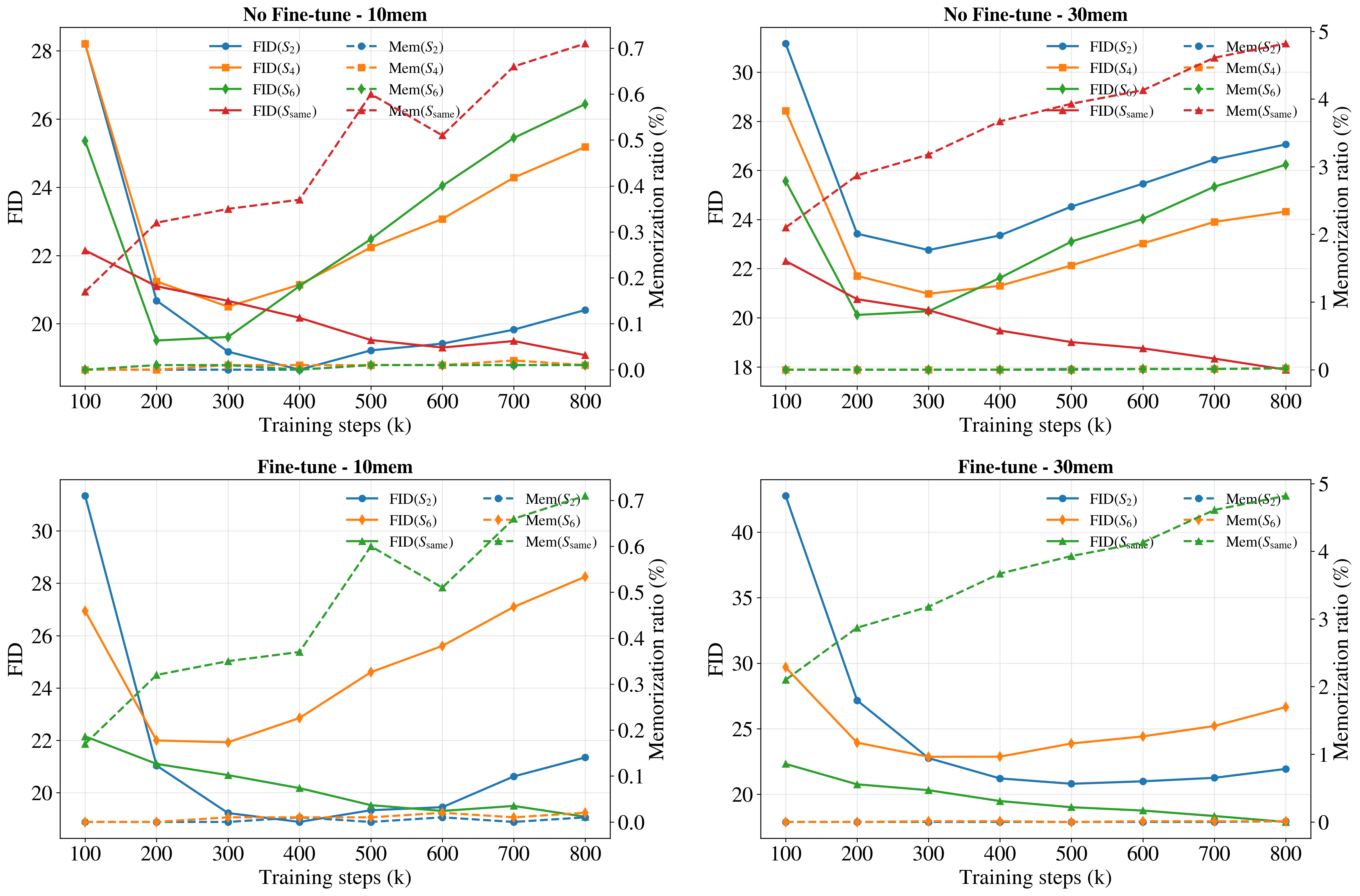}
  \caption{
  \textbf{Effect of student--teacher architectural mismatch on memorization and sample quality under two-step consistency distillation.}
  The figure reports FID (left y-axis, solid lines) and memorization ratio (right y-axis, dashed lines) as functions of training steps for different student--teacher architectural configurations.
  Results are shown for both $10\%$-memorization and $30\%$-memorization teachers.
  Across all settings, the student memorization ratio remains close to zero regardless of initialization strategy or architectural mismatch.
  However, when the student is randomly initialized (No Fine-tune), FID degrades at later training stages, indicating reduced optimization stability.
  In contrast, inheriting shared teacher parameters substantially stabilizes training and mitigates late-stage FID degradation, even when student and teacher architectures differ.}\vspace{-4mm}
  \label{fig:arch_mismatch_fid_mem}
\end{figure*}

\subsection{Effect of student--teacher architectural mismatch and initialization}
\label{sec:architecture_mismatch}

Fig.~\ref{fig:arch_mismatch_fid_mem} examines the effect of architectural mismatch and initialization strategy on consistency distillation under teachers with approximately $10\%$ and $30\%$ memorization, measured by the $l_2$ metric.
All experiments still use two-step consistency distillation on a randomly sampled subset of $5000$ training images.
We vary the depth of the student network while keeping the teacher architecture fixed, and compare two initialization strategies:
(i) \emph{no fine-tuning}, where the student is randomly initialized, and
(ii) \emph{fine-tuning}, where the student inherits all compatible parameters from the teacher, with unmatched parameters initialized randomly.
We denote by $\mathrm{S2}$, $\mathrm{S4}$, and $\mathrm{S6}$ students whose network depth is defined by $2$, $4$, and $6$ residual blocks per stage, respectively.
$\mathrm{S_{same}}$ corresponds to the setting where the student architecture exactly matches that of the teacher, whose residual blocks per stage is fixed with $4$.
All experiments use the same two-step consistency objective and identical training protocols.

A key observation is that the student memorization behavior depends strongly on \emph{architectural match}.
When the student exactly matches the teacher architecture ($\mathrm{S_{same}}$), its memorization ratio increases steadily over training, although it remains substantially lower than that of the teacher.
In contrast, when the student is either randomly initialized or architecturally mismatched (e.g., $\mathrm{S2/S4/S6}$ distilled from a fixed teacher), its memorization ratio stays near zero throughout.
This separation suggests that consistency distillation does not universally ``erase'' teacher behaviors; rather, it transfers what the student can faithfully realize under its parameterization.
When the student has sufficient expressivity and a well-aligned parameterization (the $\mathrm{S_{same}}$ setting), it can reproduce a larger portion of the teacher mapping---including a small but non-negligible fraction of teacher-specific, memorization-associated behavior.
When the student is mismatched or starts far from the teacher solution, the distillation signal becomes harder to realize precisely, and the learned solution is biased toward more conservative, distribution-level smoothing, which suppresses instance-level memorization.

The FID trends are consistent with this interpretation.
Only the architecture-matched student ($\mathrm{S_{same}}$) exhibits monotonic improvement in FID over training, indicating that the consistency objective can be optimized in a stable manner when the student can closely approximate the teacher-induced consistency function.
By contrast, for randomly initialized or mismatched students, FID improves early but degrades at later stages.
This late-stage degradation occurs despite near-zero memorization, and is indicative of an optimization bias: when exact teacher-matching is unattainable under the student parameterization, continued minimization of the consistency loss increasingly favors overly smooth and low-diversity solutions, which harms sample quality while keeping memorization low.

Consequently, these results highlight a tradeoff controlled by architectural compatibility.
Exact architectural matching enables stable quality improvement but allows partial transfer of teacher memorization, whereas mismatch or random initialization strongly suppresses memorization but can suffer from late-stage quality degradation.
A more detailed theoretical characterization of how architectural realizability and optimization dynamics jointly shape memorization transfer under consistency distillation is left for future work.


\begin{figure}[t]
  \centering
  \begin{subfigure}[t]{0.323\textwidth}
    \centering
    \includegraphics[width=\linewidth]{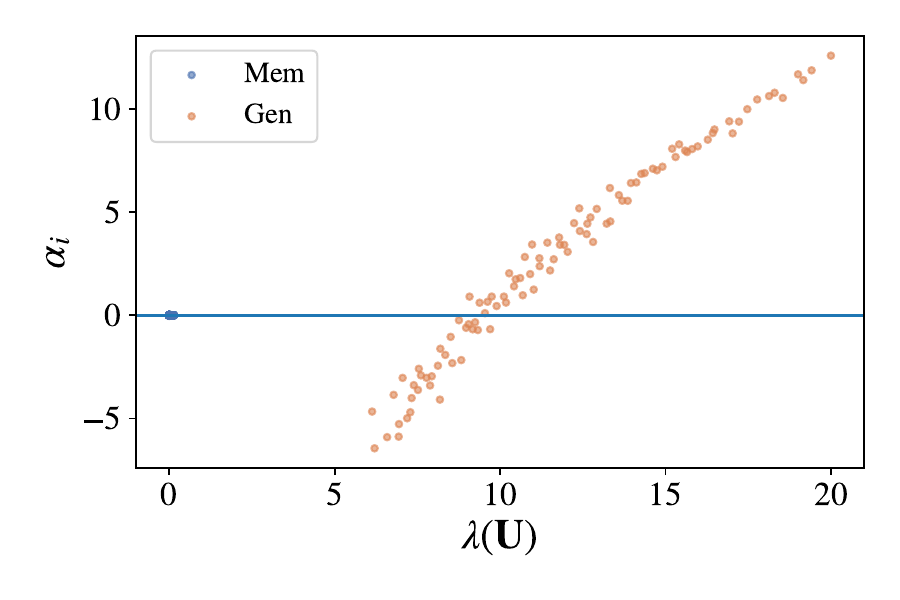}
    \caption{$\psi_p=32,\ \psi_n=16$.}
    \label{fig:alpha_vs_lambdaU_3216}
  \end{subfigure}
  \hfill
  \begin{subfigure}[t]{0.323\textwidth}
    \centering
    \includegraphics[width=\linewidth]{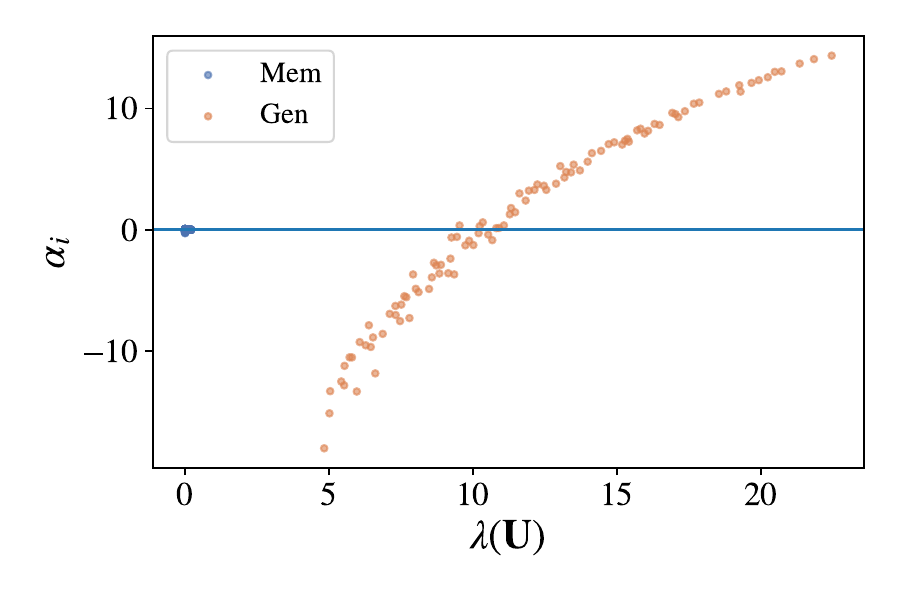}
    \caption{$\psi_p=32,\ \psi_n=8$.}
    \label{fig:alpha_vs_lambdaU_328}
  \end{subfigure}
  \hfill
  \begin{subfigure}[t]{0.323\textwidth}
    \centering
    \includegraphics[width=\linewidth]{figures/chapter5/figC_alpha_i_vs_lamU.pdf}
    \caption{$\psi_p=32,\ \psi_n=4$.}
    \label{fig:alpha_vs_lambdaU_324}
  \end{subfigure}
  \caption{
\textbf{Effect of the $\psi_p/\psi_n$ ratio on mode-wise non-isotropic CD response.}
Each point corresponds to a teacher eigenmode $u_i$, plotted against its curvature
eigenvalue $\lambda_i(\mathbf{U})$, with modes partitioned into MEM and GEN by a
fixed spectral threshold.
Across panels, the feature dimension is fixed at $\psi_p=32$, while the effective
sample ratio $\psi_n$ decreases from left to right.
As $\psi_p/\psi_n$ decreases, the fraction of GEN modes with positive signed
response $\alpha_i>0$ increases, indicating that a broader portion of the
generalization-associated spectrum receives positive non-isotropic updates,
while MEM modes remain tightly concentrated near $\alpha_i\simeq 0$.
}\vspace{-4mm}
\label{fig:alpha_vs_lambdaU_psisweep}
\end{figure}

\begin{figure}[t]
  \centering
  \begin{subfigure}[t]{0.323\textwidth}
    \centering
    \includegraphics[width=\linewidth]{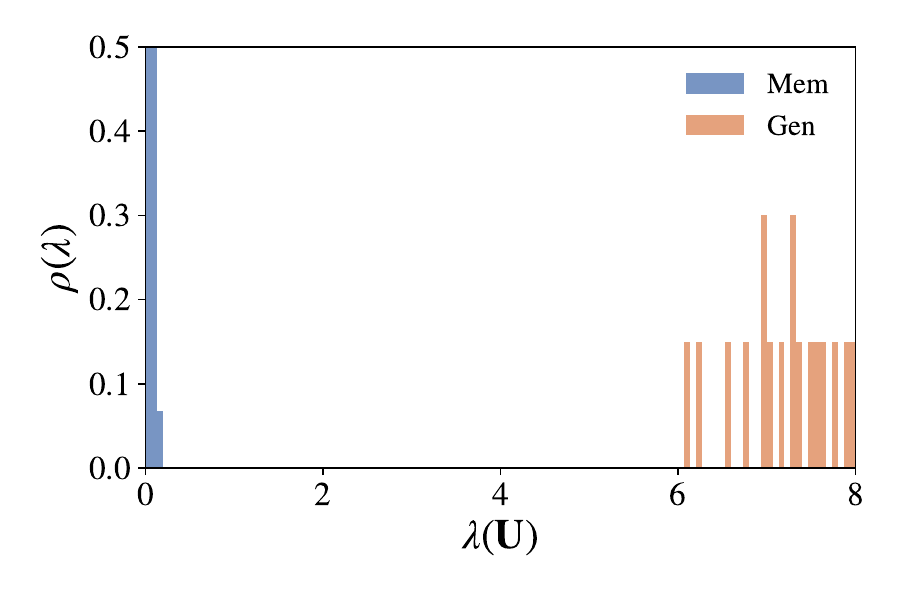}
    \caption{$\psi_p=32,\ \psi_n=16$.}
    \label{fig:U_spectrum_3216}
  \end{subfigure}
  \hfill
  \begin{subfigure}[t]{0.323\textwidth}
    \centering
    \includegraphics[width=\linewidth]{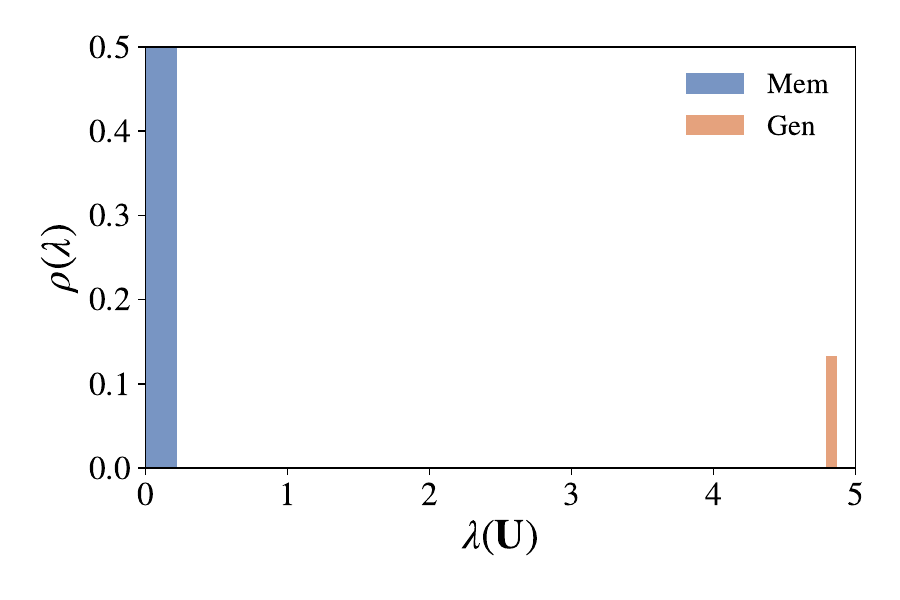}
    \caption{$\psi_p=32,\ \psi_n=8$.}
    \label{fig:U_spectrum_328}
  \end{subfigure}
  \hfill
  \begin{subfigure}[t]{0.323\textwidth}
    \centering
    \includegraphics[width=\linewidth]{figures/chapter5/figU_lambda_linear_hist.pdf}
    \caption{$\mathbf{U}$ $\psi_p=32,\ \psi_n=4$.}
    \label{fig:U_spectrum_324}
  \end{subfigure}
  \caption{
\textbf{Spectral separation of the teacher curvature matrix $\mathbf{U}$ under
different $\psi_p/\psi_n$ ratios.}
Shown are the empirical spectral densities $\rho(\lambda)$ of $\mathbf{U}$,
with modes partitioned into memorization-associated (MEM) and
generalization-associated (GEN) subspaces by a fixed threshold.
For $\psi_p=32$ and $\psi_n=16$ (left), the GEN spectrum is well separated and
concentrated at significantly larger eigenvalues than the MEM spectrum,
yielding a clear spectral gap.
As $\psi_n$ decreases (middle to right), the GEN spectrum shifts leftward and
becomes less separated from MEM modes.
This progressive loss of spectral separation provides a geometric explanation
for the change in the sign distribution of $\alpha_i$: configurations with
stronger MEM/GEN separation admit a larger fraction of GEN modes with
$\alpha_i>0$, facilitating more effective non-isotropic transfer.
}\vspace{-4mm}
\label{fig:U_spectrum_psisweep}
\end{figure}

\subsection{Why Severely Memorizing Teachers Degrade Student Generation Quality}
\label{sec:severe_mem}

As observed in Section~\ref{subsec:moderate_mem}, distillation from severely memorizing teachers can lead to noticeably degraded student FID. In this subsection, we provide a theoretical explanation for this phenomenon through the lens of our mode-wise analysis.

As observed in Section~\ref{subsec:moderate_mem}, distillation from severely memorizing teachers can lead to noticeably degraded student FID. In this subsection, we provide a theoretical explanation for this phenomenon through the lens of our mode-wise analysis. 
In particular, we use the data--feature scaling ratio $\psi_p/\psi_n$ to characterize the memorization regime underlying this behavior.
To understand how the ratio $\psi_p/\psi_n$ influences the
mode-wise behavior of non-isotropic consistency distillation, we examine how the
signed response $\alpha_i$ distributes across teacher curvature modes under
different $\psi_p/\psi_n$ configurations.
Fig.~\ref{fig:alpha_vs_lambdaU_psisweep} shows that, for fixed $\psi_p$, decreasing
$\psi_n$ (and hence increasing the ratio $\psi_p/\psi_n$) systematically enlarges
the subset of generalization-associated (Gen) modes with positive signed response
$\alpha_i>0$.
While Gen modes at very small curvature eigenvalues may remain suppressive, a
progressively larger portion of the Gen spectrum receives positive non-isotropic
updates as $\psi_n$ decreases.
Across all configurations, memorization-associated (Mem) modes remain tightly
concentrated near $\alpha_i \approx 0$, indicating that positive transfer induced by
the non-isotropic CD term is selectively allocated within the Gen subspace rather
than leaking into memorization-dominated directions.

Fig.~\ref{fig:U_spectrum_psisweep} further shows that this behavior is closely tied
to the spectral geometry of the teacher curvature matrix $\mathbf{U}$.
When $\psi_p=32$ and $\psi_n=16$, Gen modes occupy a well-separated, high-eigenvalue
region, while Mem modes remain concentrated near the origin, yielding a clear
spectral gap.
In this regime, the resolvent term $\mathbf{S}\mathbf{U}^{-1}\mathbf{S}$ primarily
suppresses low-eigenvalue directions, so that most Gen modes satisfy
$\alpha_i>0$.
As $\psi_n$ decreases, the Gen spectrum shifts toward smaller eigenvalues and the
spectral separation from Mem modes weakens, increasing the fraction of Gen modes
that fall into the suppressive regime.
This progressive loss of spectral separation directly explains the change in the
sign distribution of $\alpha_i$, linking the mode-wise effect of non-isotropic
consistency distillation to the underlying curvature spectrum of $\mathbf{U}$.

\begin{table}[t]
\centering
\caption{\textbf{Effect of the discretization granularity $N$ in consistency distillation}.}
\label{tab:discretization_effect}
\small
\begin{tabular}{lccc}
\toprule
Setting & FID & $l_2$ Mem & SSCD Mem / p95 \\
\midrule
Teacher & 22.68 & 10\% & 26.93\% / 0.8586 \\
Student $N=12$ & 29.14 & 0.18\% & 4.88\% / 0.5974 \\
Student $N=18$ & 23.76 & 0.10\% & 4.36\% / 0.5892 \\
Student $N=36$ & 17.62 & 0.02\% & 2.48\% / 0.5544 \\
\bottomrule
\end{tabular}\vspace{-2mm}
\end{table}

\subsection{Effect of Discretization Granularity}
\label{subsec:discretization_granularity}

To clarify whether the observed memorization reduction in consistency distillation is merely an artifact of coarse temporal discretization, we study the effect of the discretization granularity $N$ used during distillation. This analysis is important for distinguishing an intrinsic filtering effect of the consistency objective from a purely temporal effect caused by insufficient resolution of the teacher trajectory.
As shown in Table~\ref{tab:discretization_effect}, increasing $N$ consistently improves the student performance while further reducing memorization. This observation suggests that the reduction of memorization is not simply due to a coarse discretization failing to capture the highly curved late-stage dynamics of the teacher. Instead, finer discretization makes the local consistency constraints more faithful to the underlying trajectory, yet does not restore the teacher's memorized behavior. Empirically, better trajectory resolution leads to both improved sample quality and lower memorization, which supports the view that the filtering effect arises from the consistency distillation objective itself rather than from an accidental temporal bias induced by an overly coarse schedule.







\end{document}